\documentclass[english,twocolumn,transaction,10pt]{IEEEtran}
\usepackage{CJKutf8}
\usepackage{amsmath,amsfonts}
\usepackage{cite}
\usepackage{amsthm}
\usepackage{amssymb}
\usepackage{algorithmic}
\usepackage{algorithm}
\usepackage{array, makecell}
\newtheorem{assumption}{Assumption} 

\usepackage{float}
\usepackage{setspace} 
\expandafter\def\csname ver@subfig.sty\endcsname{}
\allowdisplaybreaks[4]

\usepackage{svg}

\usepackage{amsfonts}
\usepackage{amssymb}
\usepackage{amsthm}
\usepackage{tabularx}
\usepackage{booktabs}
\usepackage{algorithm}
\usepackage{algorithmic}

\usepackage{stfloats}
\usepackage{subcaption} 
\usepackage{graphicx}
\usepackage{caption}

\usepackage{textcomp}
\usepackage{color}
\usepackage{stfloats}
\usepackage{url}
\usepackage{verbatim}

\usepackage{graphicx}
\usepackage{cite}
\usepackage{booktabs}
\usepackage{hyperref}

\usepackage{bm}
\usepackage{multirow} 
\usepackage{tabularray}

\usepackage{float} 
\usepackage{xcolor}

\usepackage{amsmath}  
\usepackage{amsfonts} 
\usepackage{amssymb}  
\usepackage{amsthm}    
\usepackage{bm}        

\usepackage{graphicx}

\usepackage{caption}
\usepackage{subcaption}

\newtheorem{Theorem}{Theorem}
\newtheorem{lemma}{Lemma}
\newtheorem{remark}{Remark}

\title{Decentralized Federated Learning for Heterogeneous Multi-Task Semantic Communication}
\author{Lin~Yin,~Tiejun~Lv,~\IEEEmembership{Senior~Member, IEEE},~Weicai~Li,~Xi~Yu,~and~Xiaoyu~He

\thanks{Manuscript received 01 July 2025; revised 07 February 2026; accepted 13 August 2026. This paper was supported in part by the National Natural Science Foundation of China under No. 62271068. (\emph{corresponding author: Tiejun Lv}.)}

\thanks{L. Yin, T. Lv, X. Yu, and X. He are with the School of Information and Communication Engineering, Beijing University of Posts and Telecommunications (BUPT), Beijing 100876, China (e-mail: \{yinlin, lvtiejun, yusy, xiaoyuhhh\}@bupt.edu.cn).}
\thanks{W. Li is with the Center for Target Cognition Information Processing Science and Technology, and the Key Laboratory of Modern Measurement and Control Technology, Ministry of Education, both at Beijing Information Science and Technology University, Beijing, China (e-mail: liweicai@bistu.edu.cn).}
}

\begin{document}

\captionsetup[figure]{font={small}, name={Fig.}, labelsep=period}
	\maketitle
 
\begin{abstract}
Collaborative training in distributed semantic communication (DSC) networks typically relies on decentralized federated learning (DFL). However, pushing topology-agnostic aggregation into heterogeneous, multi-task environments creates a fundamental bottleneck: it drives negative transfer and over-consensus bias (OCB). This paper introduces a personalized DSC framework that cuts off this cross-task interference. At the node level, a policy-driven multi-path routing mechanism separates task-specific features from shared representations to preserve local fidelity. Across the network, we deploy a ``communication-while-aggregation'' protocol. It calibrates a column-stochastic consensus matrix using task affinities. This limits the system to absorbing complementary knowledge while actively blocking mismatched parameter updates. To bound the convergence, we derive a unified Lyapunov drift analysis. We reveal a strict U-shaped trade-off: deeper topological mixing reduces variance but amplifies structural OCB. Resolving this tension yields a closed-form expression for the optimal aggregation depth. We evaluate the proposed framework on NYU-v2, where the results reveal a clear trade-off between insufficient aggregation and excessive topological mixing. At the analytically derived optimal aggregation depth, our method achieves a $4.77\%$ global relative improvement over the no-aggregation baseline and outperforms decentralized FedAvg, FedAMP, and heuristic \textit{max} aggregation.
 We further evaluate the framework on Taskonomy and imperfect wireless links to examine the effects of network-size variation and wireless-link reliability.

\end{abstract}


\section{Introduction}

Modern wireless networks increasingly execute data-intensive inference tasks directly at the edge. To support this shift, semantic communication (SC) bypassing traditional bit-level transmission has emerged as a practical solution to bandwidth bottlenecks. By extracting and transmitting only task-relevant representations, SC drastically cuts communication overhead~\cite{ 10431795}. Yet, most existing SC architectures assume single-task environments~\cite{8723589, 11514062}. Real-world edge applications are rarely this simple; they require concurrent execution of diverse, heterogeneous tasks, necessitating multi-task learning (MTL)~\cite{8954326, 8954221}. Centralized MTL orchestration dictates heavy communication costs, strict latency bounds, and severe privacy risks. Decentralized federated learning (DFL) deployed over peer-to-peer meshes avoids this central bottleneck entirely. It forces collaboration directly to the edge~\cite{11396093}.

\subsection{Related Work and Motivations}

Distributed optimization and task-oriented representations must merge to support edge intelligence. Initial SC research relied heavily on single-modal deep joint source-channel coding (DJSCC)~\cite{8723589, 9953110}. The field now targets multi-task SC (MT-SC) designs. Recent architectures execute multiple downstream tasks through graph attention mechanisms~\cite{10820866}, scalable feature ranking~\cite{10095672}, and shared codebooks~\cite{10431795}. Asynchronous multi-task semantic communication provides another route to extract task-independent knowledge~\cite{10273382}. A critical blind spot remains: almost all these MT-SC systems presuppose a central server. They strip away the realities of decentralized data distributions and the dynamic topological limits inherent to wireless edge networks.

DFL protocols take the opposite approach. They cut out the orchestrator via peer-to-peer gossip averaging. The presence of non-independent and identically distributed (non-IID) data severely hinders the convergence of global consensus, thereby necessitating local adaptation through personalized federated learning (PFL). To mitigate the performance degradation induced by such data heterogeneity, recent literature explores clustered subgraphs~\cite{11271864}, model and gradient decoupling~\cite{11194117}, robust decentralized personalization ~\cite{10924413} and over-the-air computation~\cite{11421483}. Security- and resource-aware decentralized learning has also been studied in edge systems, where asynchronous reinforcement federated learning and task offloading are used to support secure and privacy-preserving operation~\cite{NetworkSecurity}. This line of work mainly concerns security, privacy, and offloading decisions, whereas our setting is shaped by semantic mismatch among heterogeneous task portfolios. In heterogeneous multi-task learning, task compatibility can influence whether shared training is beneficial or leads to interference. Studies on visual task transferability, task grouping, and gradient conflict provide related evidence for using task compatibility to regulate knowledge sharing~\cite{Zamir, standley2020tasks, Yu2020Gradient}. Similarity-guided collaboration has also been used in PFL, where aggregation gives larger influence to related clients rather than enforcing a single averaged model~\cite{Huang2021Personalized}. These studies motivate our use of task-level semantic similarity as a criterion for guiding peer-to-peer aggregation in decentralized multi-task distributed SC (DSC) networks.

Mixing mismatched tasks over a decentralized topology guarantees severe negative transfer, client drift, and over-consensus bias (OCB)~\cite{10982277}. Standard PFL fixes like clustering or meta-learning~\cite{ You2025AFR} break down in SC environments. They demand massive communication payloads. Adaptive aggregation schemes~\cite{xiao2026divergence, 10658480} routinely force clients to broadcast dense, high-dimensional neural parameters. The edge physical layer simply cannot support this bandwidth overhead. 

Recent literature has attempted to narrow this communication bottleneck through optimizations at the physical and network layers. For example, ~\cite{10542235} and ~\cite{10965802} study decentralized federated learning over imperfect channels to mitigate transmission impairments. In the semantic communication domain, recent works investigate robust semantic digital-to-analog conversion for discrete quantization~\cite{10985906} and develop empirical performance models for resource allocation~\cite{11573154}. In addition,~\cite{11263916} incorporates semantic encoding into a hierarchical federated architecture to reduce satellite communication overhead.
Despite these advances, existing frameworks remain insufficient for multi-task edge environments. Many rely on semi-centralized multi-tier architectures~\cite{11263916}, while others treat the underlying neural model as a monolithic object during communication and aggregation~\cite{10981930,10985906}. As a result, they cannot explicitly separate task-specific semantic components from the shared representation space~\cite{chen2024fedbone}.

Current designs largely decouple communication-efficient semantic compression from personalized decentralized training. This limitation suggests that robust multi-task edge intelligence requires personalization to be incorporated directly into the semantic compression process. In particular, neighborhood updates should be selectively screened and reweighted so that mismatched semantic information is suppressed while local task fidelity is preserved.

\subsection{Contributions and Organization}
Solving this structural disconnect requires moving personalization from the federated averaging step into the semantic representation process. To achieve this, we propose a multi-path semantic encoder with a learnable feature routing policy. Rather than relying on a monolithic representation, the proposed encoder balances shared semantic extraction with task-specific feature preservation. At the network level, these learned routing policies support a similarity-aware decentralized aggregation scheme, where each client evaluates semantic alignment with its neighbors based on task-level semantic similarity and adjusts the aggregation weights accordingly. The resulting framework treats the semantic transceiver as the interface for exchanging task-relevant features, while the main optimization is placed on routing-policy learning and similarity-aware neighbor aggregation. As a result, mismatched updates are suppressed, whereas semantically aligned updates are assigned greater influence.

In summary, our primary contributions are threefold:
\begin{itemize}
    \item A distributed, personalized learning framework for heterogeneous semantic communication: instead of relying on monolithic full-model exchange, collaborative modeling is organized around the multi-path encoder that extracts task-relevant semantic representations. This design reduces the communication burden associated with dense parameter sharing while retaining task-specific modeling capacity at local clients.
    
    \item A similarity-aware aggregation protocol designed specifically to counter OCB and negative transfer: we utilize the task-specific policy vectors to calibrate the topological weights during training. Consequently, the network blocks mismatched updates and forces compatible peers into tighter cooperation. 
    
    \item A rigorous convergence proof for the ``communication-while-aggregation'' protocol: using a unified Lyapunov drift analysis, we uncover a fundamental U-shaped tension; i.e., the variance reduction gained from topology mixing is in direct conflict with structural OCB. We resolve this tension by deriving a closed-form solution for the optimal aggregation depth $J^{*}$, establishing the exact point where the steady-state error floor hits its absolute minimum.
\end{itemize}

Empirical evaluations on the NYU-v2 benchmark validate the existence of these theoretical boundaries. If we restrict the network to a shallow mixing regime ($J=1$), the similarity-aware routing still manages a $2.11\%$ global relative improvement. It also acts as a safeguard, preventing the highly sensitive depth estimation task from collapsing entirely. Allowing the system to operate at the analytically derived optimal depth unlocks the full potential of the framework, yielding a $4.77\%$ overall relative gain. Looking at individual tasks, Semantic segmentation jumps by $8.04\%$ and depth estimation improves by $3.31\%$. These margins confirm a critical design hypothesis: explicitly blocking topology-agnostic mixing is mandatory to survive cross-task interference at the edge. We further include Taskonomy experiments and communication-quality simulations to examine the effects of network-size variation and imperfect wireless links.

The remainder of the paper proceeds as follows. Section \uppercase\expandafter{\romannumeral2} sets up the system model. Section \uppercase\expandafter{\romannumeral3} formulates the decentralized training protocol alongside the core optimization objectives. The theoretical convergence analysis, culminating in the derivation of the optimal aggregation depth, is presented in Section \uppercase\expandafter{\romannumeral4}. Section \uppercase\expandafter{\romannumeral5} discusses the experimental findings, and Section \uppercase\expandafter{\romannumeral6} draws the final conclusions.

\begin{table}[htbp]
\caption{Summary of Key Notations}
\label{tab:notations}
\centering
\small
\setlength{\tabcolsep}{4pt}
\renewcommand{\arraystretch}{1.08}
\begin{tabularx}{\columnwidth}{@{}>{\raggedright\arraybackslash}p{0.22\columnwidth}>{\raggedright\arraybackslash}X@{}}
\toprule
\textbf{Notation} & \textbf{Definition} \\
\midrule
$i,\mathcal{N}$ & Client index, and set of clients in the network. \\
$t,\mathcal{T}$ & Task index, and set of global semantic tasks. \\
$\mathcal{T}_{i}$ & Personalized task set of client $i$. \\
$n,\mathcal{N}_{i}$ & One-hop neighbor index, and set of neighbors of client $i$. \\
$\mathbf{u}_{t}$ & Learnable binary policy vector for task $t$. \\
$s_{t,t'}$ & Semantic similarity between task $t$ and task $t'$. \\
$S_{i,n}$ & Overall semantic alignment between client $i$ and its neighbor $n$. \\
$\mathcal{L}_{i,t}$ & Task-level local loss of client $i$ on task $t$. \\
$\mathcal{L}_{i}$ & Client-level local loss of client $i$. \\
$\mathcal{L}$ & Global network-wide loss. \\
$r,R$ & Communication round index and total number of communication rounds. \\
$e,E$ &  Training epoch index and total epochs in one round. \\
$j,J$ & Aggregation depth index and total number of aggregation steps. \\
$\psi_{i,r,j}$ & Model parameters of client $i$ after the $j$-th aggregation step in round $r$. \\
$w_{i,n},\,\mathbf{W}$ & Aggregation weight between client $i$ and neighbor $n$, and global aggregation weight matrix. \\
$E_{\mathrm{cons},r,J}$ & Round-wise consensus error after $J$ aggregation steps in communication round $r$. \\
$\Phi_{r,J}$ & Composite Lyapunov potential function. \\
$\Omega_*^2$ & Over-consensus bias. \\
$J^*$ & Optimal aggregation depth. \\
\bottomrule
\end{tabularx}
\end{table}

\section{System Model}
We detail the architecture underpinning our DSC framework. The network functions as a decentralized mesh. Clients bypass central coordination entirely, relying instead on direct neighborhood exchanges to jointly train their local multi-path encoders and task-specific decoders.

\begin{figure}[htbp]
    \centering
    \includegraphics[width=1\linewidth, trim={0cm 0cm 0cm 0cm}, clip]{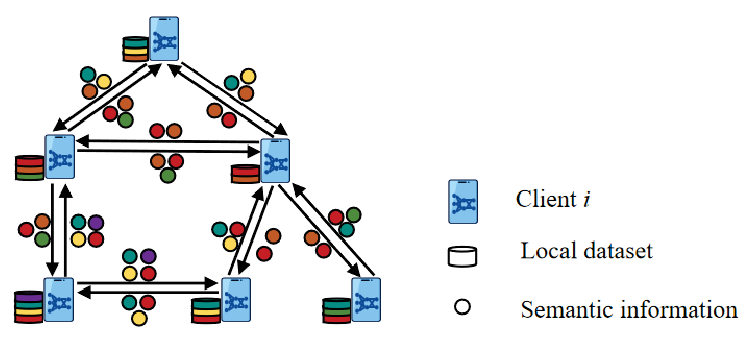}
    \vspace{-0.5cm} 
    \caption{The proposed distributed semantic communication framework. The distinct colors indicate different tasks and semantic information.}
    \label{fig:communication_flow}
\end{figure}

\subsection{Distributed Semantic Communication Model}\label{section DSC model}
The DSC network comprises a set of clients $\mathcal{N}=\{1,\cdots,N\}$ forming a mesh topology, and a global semantic task set $\mathcal{T}=\{1,\cdots,T\}$. Each client $i \in \mathcal{N}$ is assigned a personalized semantic task portfolio $\mathcal{T}_{i} \subseteq \mathcal{T}$ and possesses an independent local dataset $\mathcal{D}_{i}=\cup_{t \in \mathcal{T}_i} \mathcal{D}_{i,t}$, where $\mathcal{D}_{i,t}=\{(x,y_{t})\}$ contains the raw input $x$ and the ground truth labels $y_{t}$ specific to task $t$. Client $i$ communicates exclusively with its one-hop neighbor set $\mathcal{N}_{i}$. The one-hop neighbor set $\mathcal{N}_i$ is specified by a distance-constrained communication graph. Each edge represents a feasible peer-to-peer link between two clients. Information from non-neighboring clients is not obtained through direct long-range broadcasting, but through repeated one-hop aggregation over the connected graph. This setting keeps the decentralized training process local while allowing model information to propagate across the network over multiple aggregation steps.

Rather than reconstructing raw bit streams, the localized SC model extracts and transmits only task-relevant features. This design explicitly bypasses conventional bandwidth bottlenecks. As Fig.~\ref{fig:communication_flow} illustrates, the network assigns a heterogeneous multi-task portfolio to each client operating on a strictly local dataset. Clients completely avoid transmitting raw data or dense model weights. They locally encode task-specific semantic representations and exchange them directly across peer-to-peer mesh links, eliminating the need for central coordination.

Consider a raw input image $x_i \in \mathbb{R}^{C_{\mathrm{in}} \times W_{\mathrm{in}} \times H_{\mathrm{in}}}$ at client $i$. The joint encoder $f_{t}^{\mathrm{enc}}(\cdot)$ parameterized by $\theta_{i,t}^{E}$ generates the task-relevant latent representation $X_{i,t}$ for task $t$ as
\begin{equation}
    X_{i,t}=f_{t}^{\mathrm{enc}}(x_i; \theta_{i,t}^{E}), 
    \quad 
    X_{i,t} \in \mathbb{R}^{C_{\mathrm{out}} \times W_{\mathrm{out}} \times H_{\mathrm{out}}}.
\end{equation}
Here, $C_{\mathrm{in}}$, $W_{\mathrm{in}}$, and $H_{\mathrm{in}}$ denote the channel number, width, and height of the input image, respectively, while $C_{\mathrm{out}}$, $W_{\mathrm{out}}$, and $H_{\mathrm{out}}$ denote those of the latent semantic feature.

The latent tensor $X_{i,t}$ is mapped into a sequence of complex channel symbols $\mathbf{x}_{i,t} \in \mathbb{C}^{S}$, where $S$ denotes the available symbol budget of the edge link. The symbols are transmitted from client $i$ to its one-hop neighbor $n$ over the decentralized communication graph. During neighborhood exchange, transmissions follow a time division multiple access (TDMA)-based link schedule. To construct the schedule, we build a conflict graph over the feasible one-hop links: two links are treated as conflicting if they share a transmitter or receiver, or if one transmission may interfere with the receiver of the other link under the adopted interference range. A graph-coloring rule then assigns conflicting links to different time slots, following coloring-based link scheduling for TDMA transmissions in wireless networks~\cite{Gandham2008LinkScheduling}. Under this scheduled one-hop exchange, an active link is modeled with the desired signal, channel fading, and additive white Gaussian noise (AWGN), while conflicting links are deactivated in the same time slot.

For a scheduled transmission over link $(i,n)$, the received signal vector at client $n$ is modeled as
\begin{equation} 
    \mathbf{y}_{n,i,t} = \mathbf{H}_{n,i}\mathbf{x}_{i,t} + \mathbf{z}_{n,i,t}, 
\end{equation}
where $\mathbf{H}_{n,i} \in \mathbb{C}^{S \times S}$ represents the equivalent channel matrix of link $(i,n)$, and $\mathbf{z}_{n,i,t} \sim \mathcal{CN}(\mathbf{0}, \sigma_{c}^{2}\mathbf{I})$ is the AWGN vector.

Assuming that client $n$ has channel state information for the scheduled link, a zero-forcing equalizer estimates the transmitted symbols as
\begin{equation}
    \hat{\mathbf{x}}_{n,i,t}
    =
    (\mathbf{H}_{n,i}^{\mathrm{H}}\mathbf{H}_{n,i})^{-1}
    \mathbf{H}_{n,i}^{\mathrm{H}}\mathbf{y}_{n,i,t}
    =
    \mathbf{x}_{i,t}+\tilde{\mathbf{z}}_{n,i,t},
\end{equation}
where $\mathbf{H}_{n,i}^{\mathrm{H}}$ denotes the Hermitian transpose, and $\tilde{\mathbf{z}}_{n,i,t}$ is the amplified noise after equalization.

The encoder and decoder for task $t$ are trained as a paired semantic transceiver at client $i$. After training, the corresponding decoder parameters $\theta_{i,t}^{D}$ are made available to the intended receiving neighbor for semantic reconstruction. Client $n$ reshapes $\hat{\mathbf{x}}_{n,i,t}$ into the feature tensor $\hat{X}_{n,i,t}$ and reconstructs the task prediction as
\begin{equation}
    \hat{y}_{n,i,t}=f_{t}^{\mathrm{dec}}(\hat{X}_{n,i,t}; \theta_{i,t}^{D}),
\end{equation}
where $\theta_{i,t}^{D}$ denotes the task-specific decoder trained at client $i$ and used by client $n$ to decode the received semantic feature.

The scheduled one-hop link is further associated with packet-level reliability. For link $(i,n)$, we use the packet loss probability $p_{n,i}^{\mathrm{loss}}$ to characterize the reliability of packet delivery. This probability reflects the link-level signal-to-noise ratio (SNR) condition and is determined by the nominal SNR setting, link distance, and packet length~\cite{10542235}. A packet-level outage or decoding failure is counted as an unsuccessful packet delivery.

Let $K_{n,i}^{\mathrm{pkt}}$ denote the number of packets transmitted over link $(i,n)$, which is determined by the message size and the packet length. When a packet is lost, a retransmission is triggered until successful delivery. Under this latency model, packet loss does not alter the decoded semantic feature after successful reception, but it increases the communication time. If $\ell_{\mathrm{pkt}}$ denotes the duration of one packet transmission, the expected link transmission latency is
\begin{equation}
    \ell_{n,i}^{\mathrm{link}}
    =
    \frac{K_{n,i}^{\mathrm{pkt}}\ell_{\mathrm{pkt}}}
    {1-p_{n,i}^{\mathrm{loss}}}.
\end{equation}
This packet-level abstraction allows link reliability to affect decentralized aggregation through retransmission latency, without introducing a detailed physical-layer modulation, coding, or outage analysis.

\subsection{Design of Multi-Task Joint Source-Channel Coding}
\subsubsection{Policy-Driven Multi-Path Encoder}
To enable heterogeneous multi-task extraction without instantiating independent backbones, our joint encoder employs a policy-driven feature routing architecture~\cite{sun2020adasharelearningshareefficient}. The shared backbone comprises $K$ sequential computational blocks. We define a learnable binary policy vector $\mathbf{u}_{t} = [u_{t,1}, \dots, u_{t,K}]^{\top} \in \{0,1\}^K$ for each task $t$, which dictates the execution path. 

For the $t$-th task of client $i$, the latent representation evolves recursively through the $K$ blocks according to
\begin{equation}
    \boldsymbol{\alpha}_{i,t,k} = u_{t,k} \cdot \mathcal{F}_{i,k}(\boldsymbol{\alpha}_{i,t,k-1}) + \boldsymbol{\alpha}_{i,t,k-1},
\end{equation}
for all $k \in \{1,\cdots,K\}$, where $\boldsymbol{\alpha}_{i,t,0} = x$ is the initial input, and $\mathcal{F}_{i,k}(\cdot)$ represents the transfer function of the $k$-th shared block parameterized by the local client. The indicator $u_{t,k}$ dynamically controls whether the $k$-th block is activated or bypassed. The final output $\boldsymbol{\alpha}_{i,t,K}$ corresponds to the extracted semantic feature $X_t$.

This mechanism ensures that the active encoder parameters $\theta_{i,t}^{E}$ are a functionally restricted subset of the global parameter space $\mathcal{F}_{i}=\{\mathcal{F}_{i,k}\}_{k=1}^K$, governed by the task-specific policy $\mathbf{u}_t$. By disentangling the execution paths, the network prevents gradient interference and structural OCB among disjoint tasks. As illustrated in Fig.~\ref{fig:Encoder}, a policy-driven multi-path encoder adaptively extracts task-specific semantic representations by routing the input through shared computational blocks. Subsequently, independent task-specific decoders reconstruct the final predictions for their corresponding tasks. 

\begin{figure}[htbp]
    \centering
    \includegraphics[width=\linewidth]{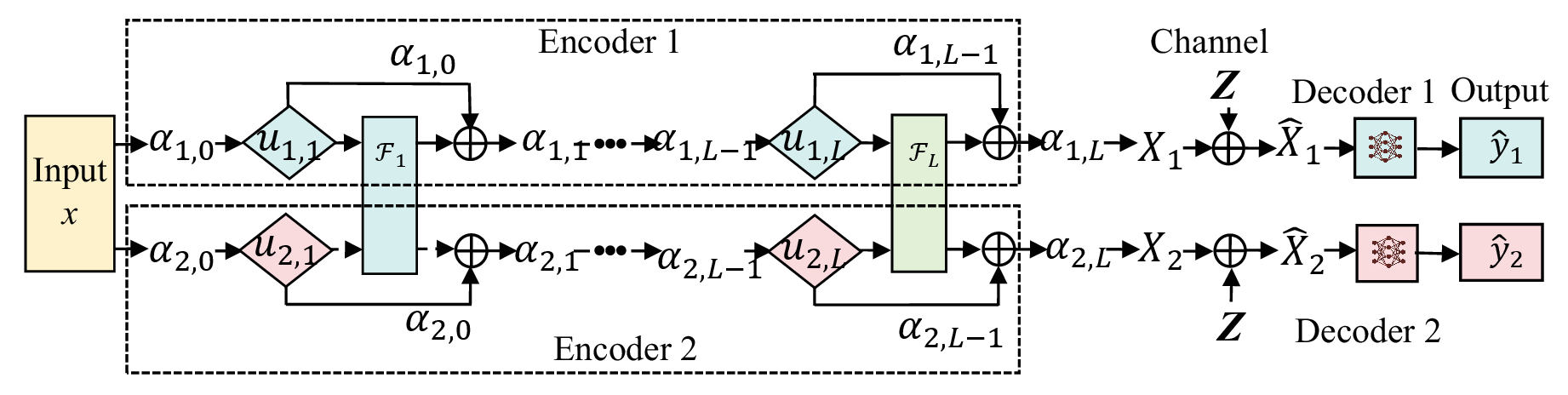}
    \caption{Schematic of the policy-driven multi-path semantic encoding and task-specific decoding architecture for heterogeneous multi-task execution.}
    \label{fig:Encoder}
\end{figure}

The routing policy also provides a compact descriptor of task-specific use of the shared encoder. Since $u_{t,k}$ records whether task $t$ activates the $k$-th shared block, two tasks with more overlapping active blocks tend to update more overlapping parts of the shared backbone during local training. This overlap serves as a lightweight indicator of compatibility in the shared representation space. Consequently, we quantify inter-task affinity by defining the semantic similarity between any two tasks $t \in \mathcal{T}_{i}$ and $t' \in \mathcal{T}_{n}$ as the cosine similarity of their policy vectors expressed as
\begin{equation}
    s_{t,t'}=\frac{\mathbf{u}_{t}^{\top} \mathbf{u}_{t'}}{\|\mathbf{u}_{t}\| \|\mathbf{u}_{t'}\|}.
\end{equation}

Under this binary policy representation, the cosine score can be understood as a normalized measure of active-path overlap. Its numerator counts the shared activated blocks between two tasks, while the normalization avoids favoring tasks merely because they activate more blocks. Thus, the metric provides a lightweight indicator of task compatibility within the proposed multi-path encoder, and its interpretation is based on how tasks select and share encoder blocks.

And the overall semantic alignment between client $i$ and its neighbor $n$ is evaluated using symmetric best matching as
\begin{equation}\label{eq:client_similarity}
    S_{i,n} = \frac{1}{2|\mathcal{T}_i|} \sum_{t \in \mathcal{T}_i} \max_{t' \in \mathcal{T}_n} s_{t,t'} + \frac{1}{2|\mathcal{T}_n|} \sum_{t' \in \mathcal{T}_n} \max_{t \in \mathcal{T}_i} s_{t,t'},
\end{equation} 
where $S_{i,n}$ guides decentralized aggregation weights, prioritizing clients with overlapping pathways while filtering out those with disjoint task portfolios.

\subsubsection{Task-Specific Decoders}
While the encoder employs a shared routing backbone, each task $t$ maintains an independent, lightweight DeepLab-style dilated convolutional decoder. The decoder exploits multi-scale contextual features via parallel branches, which are linearly fused along the channel dimension. Because the parameters $\theta_{i,t}^{D}$ are strictly task-specific, they are decoupled from the neighbor aggregation process, ensuring high local task fidelity.

\subsection{Personalized and Global Optimization Objectives}
For client $i$, the empirical loss on task $t$ is evaluated over its corresponding local dataset $\mathcal{D}_{i,t}$, as given by
\begin{equation}
    \mathcal{L}_{i,t}(\lambda_{i,t},\mathcal{D}_{i,t},\theta_{i,t})=\frac{1}{|\mathcal{D}_{i,t}|}\sum_{(x,y_{t})\in \mathcal{D}_{i,t}} \lambda_{i,t} L_{t}(\hat{y}_{i,t},y_{t}),
\end{equation}
where $\lambda_{i,t}$ denotes the task-specific weighting coefficient that governs the relative importance of task $t$, $L_{t}(\cdot, \cdot)$ is the task-specific criterion (e.g., cross-entropy for segmentation), and $\theta_{i,t}$ encompasses both the localized encoder pathway and the task-specific decoder parameters. While all parameters in $\theta_{i,t}$ are jointly updated during local gradient descent, only the shared encoder routing backbone participates in the decentralized network-wide consensus protocol.

Let $\theta_{i} = \mathcal{F}_{i}$ denote the entirety of the aggregatable encoder backbone on client $i$. By treating the localized decoders as conditionally optimized internal variables, the aggregate local objective for client $i$ is formulated with respect to the shared backbone as
\begin{equation}\label{eq:local_loss}
    \mathcal{L}_{i}(\Lambda_{i},\mathcal{D}_{i},\theta_{i})=\sum_{t\in \mathcal{T}_{i}} \mathcal{L}_{i,t}(\lambda_{i,t},\mathcal{D}_{i,t},\theta_{i,t}),
\end{equation}
where $\Lambda_{i}=\{\lambda_{i,t}|t \in \mathcal{T}_{i}\}$ collects the predefined importance weights for the local task portfolio.

Building upon the local objectives, the generalized global loss function over the entire mesh network is defined as the uniform average of all client-specific losses given by
\begin{equation}
    \mathcal{L}(\Lambda,\mathcal{D},\boldsymbol{\theta})=\frac{1}{N}\sum_{i\in \mathcal{N}}\mathcal{L}_{i}(\Lambda_{i},\mathcal{D}_{i},\theta_{i}),
\end{equation}
where $\Lambda=\{\Lambda_{i}|i\in \mathcal{N}\}$, $\mathcal{D}=\{\mathcal{D}_{i}|i\in \mathcal{N}\}$, and $\boldsymbol{\theta}=\{\theta_{i}|i\in \mathcal{N}\}$ aggregates the shared parameters across all clients.

The DFL system optimizes two interdependent objectives. First, the localized personalization objective seeks the optimal parameters $\theta_{i}^*$ that minimize client $i$'s empirical risk formulated as
\begin{equation}\label{eq:personal_obj}
    \theta_{i}^*=\underset{\theta_{i}}{\arg \min} \,\mathcal{L}_{i}\left(\Lambda_{i},\mathcal{D}_{i},\theta_{i}\right), \quad \forall i \in \mathcal{N}.
\end{equation}
Secondly, the network-level consensus objective seeks the global minimizer $\boldsymbol{\theta}^*$ formulated as
\begin{equation}\label{eq:global_obj}
    \boldsymbol{\theta}^*=\underset{\boldsymbol{\theta}}{\arg \min} \, \mathcal{L}(\Lambda,\mathcal{D},\boldsymbol{\theta}).
\end{equation}
The topological tension between the local minimizers $\theta_{i}^*$ and the global minimizer $\boldsymbol{\theta}^*$ is regulated by the within-round aggregation depth $J$, a phenomenon we analytically characterize in the subsequent convergence analysis.
\section{Decentralized Training for the DSC Model}
Every client within the fully decentralized DSC network independently tackles a personalized empirical risk minimization problem. Consequently, the sharing of knowledge across the network is dictated exclusively by peer-to-peer communications. Fig.~\ref{fig:system model} breaks down our DFL training paradigm. We enforce an alternating sequence for every communication round. Local data drives the initial gradient-based semantic updates. Immediately after, the client pulls in peer parameters, fusing them through dynamic, similarity-aware weighting.

\begin{figure}[htbp]
    \centering
    \includegraphics[width=\linewidth]{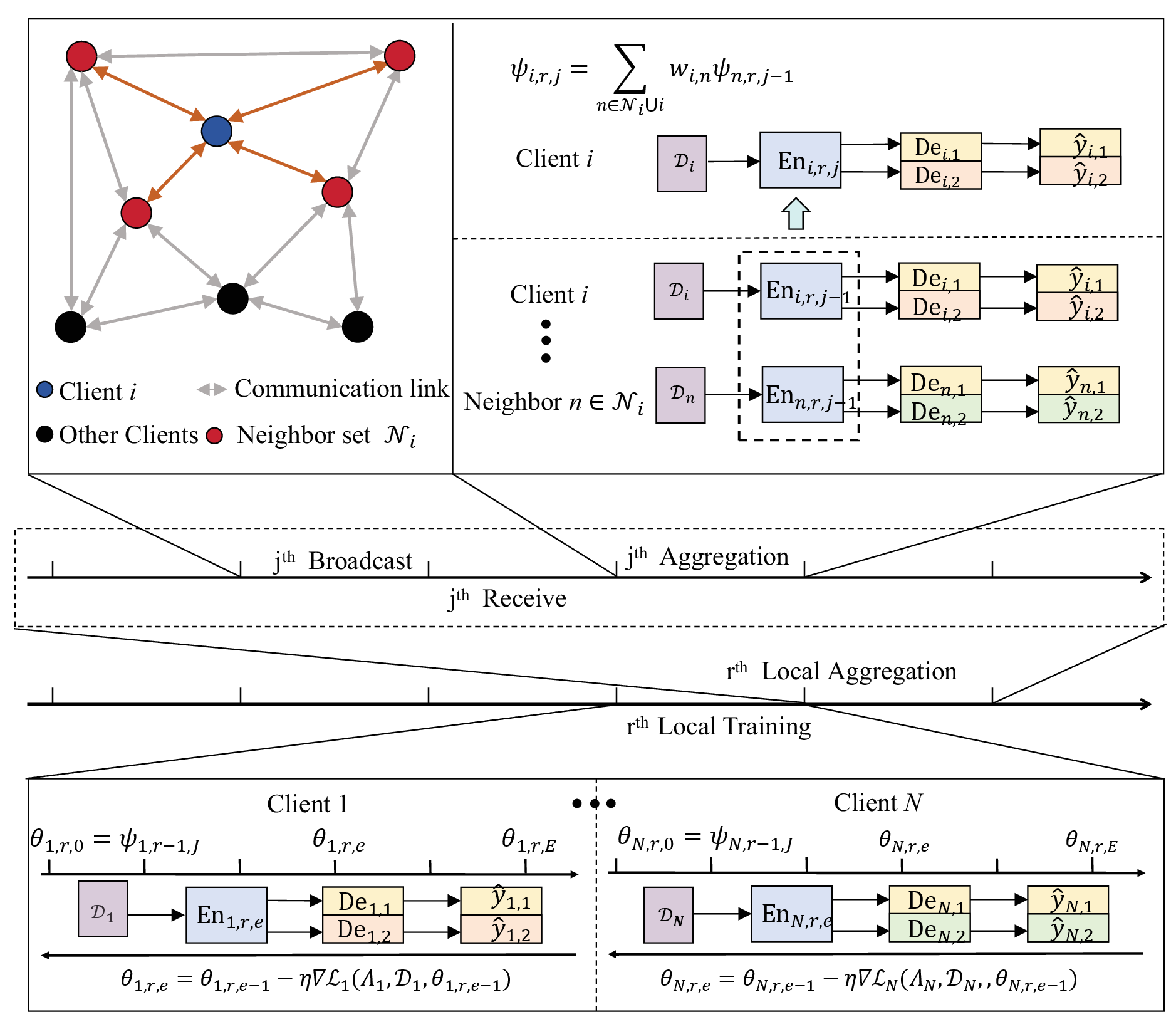}
    \caption{Workflow of the proposed ``communication-while-aggregation'' protocol for the distributed personalized SC system.}
    \label{fig:system model}
\end{figure}

\subsection{Local Model Training}
In the DSC network, clients perform parallel training by iteratively updating aggregatable model parameters $\theta_{i}$ via stochastic gradient descent (SGD). 

Each client $i$ performs $R$ communication rounds during the training process. In the $r$-th communication round, each client $i$ executes $E$ local epochs based on its local dataset $\mathcal{D}_{i}$. The local training starts with its locally aggregated model from the $(r-1)$-th communication round. During local training, the model at the $e$-th epoch, denoted by $\theta_{i,r,e}$, is updated via
\begin{equation}\label{eq:local_model_update}
    \theta_{i,r,e}=\theta_{i,r,e-1}-\eta \nabla \mathcal{L}_{i}(\Lambda_{i},\mathcal{D}_{i},\theta_{i,r,e-1}),
\end{equation}
where $\eta$ is the learning rate, and $\nabla(\cdot)$ denotes the stochastic gradient. Consequently, the output parameter of the local training phase in round $r$ after $E$ epochs is given by
\begin{equation}\label{eq:local_output}
    \theta_{i,r,E} = \theta_{i,r,0} - \eta {\sum}_{e=0}^{E-1} \nabla \mathcal{L}_{i}(\Lambda_{i},\mathcal{D}_{i},\theta_{i,r,e}).
\end{equation}

To quantify the structural divergence inherent in heterogeneous decentralized training, we introduce an idealized reference trajectory. Assuming client $i$ evaluates gradients at the instantaneous network average $\bar{\theta}_{r, e} = \frac{1}{N}\sum_{i \in \mathcal{N}} \theta_{i,r,e}$ rather than its local state, the resulting synchronized reference state $\hat{\theta}_{i,r, E}$ is formulated as
\begin{equation}\label{eq:ideal_local_training}
    \hat{\theta}_{i,r, E}=\theta_{i,r,0}-\eta {\sum}_{e=0}^{E-1}\nabla \mathcal{L}_{i}(\Lambda_{i},\mathcal{D}_{i},\bar{\theta}_{r,e}).
\end{equation}
This reference model establishes a rigorous mathematical anchor to bound the consensus error and characterize the aggregation trade-offs in the subsequent theoretical analysis.

\subsection{Local Model Aggregation}
We adopt a ``communication-while-aggregation'' protocol as our decentralized aggregation scheme. This peer-to-peer gossip protocol is scalable, fault-tolerant, and inherently bypasses the central bottleneck. During the aggregation phase of each communication round, each client transmits its task-specific policy vectors together with the current shared encoder parameters to its one-hop neighbors. The received shared encoder parameters are then integrated over $J$ consecutive peer-to-peer aggregation steps according to
\begin{equation}\label{eq:aggregation}
    {\psi}_{i,r,j}={\sum}_{n \in \mathcal{N}_{i} \cup \{i\}}w_{i,n}\psi_{n,r,j-1},
\end{equation}
where ${\psi}_{i,r,0} = \theta_{i,r,E}$, and $w_{i,n}$ is the aggregation weight denoting the fraction of model information transmitted from neighbor $n$ to client $i$.  

To prioritize semantically aligned updates while preserving the exact global parameter average during communication, the raw similarities $S_{i,n}$ defined in Eq.~\eqref{eq:client_similarity} are converted into normalized aggregation coefficients. Specifically, we apply a column-wise softmax operation over the outgoing edges of each sender $n$, expressed as
\begin{equation}\label{eq:softmax}
    w_{i,n}=\frac{\exp(S_{i,n}/\tau)}{\sum_{m\in\mathcal{N}_{n}\cup\{n\}}\exp(S_{m,n}/\tau)},
\end{equation}
where $\tau>0$ is a temperature parameter, and the self-similarity $S_{n,n}$ is defined as the maximum possible affinity metric. This guarantees that $\sum_{i\in\mathcal{N}_{n}\cup\{n\}} w_{i,n} = 1$ for any client $n$. 
In this protocol, cross-client knowledge enters a local model through neighborhood aggregation. Following the usual interpretation of negative transfer~\cite{Wang_2019_CVPR}, aggregation-induced negative transfer refers to the case where received neighbor updates are poorly aligned with a client's local task requirements and degrade its local task performance. This degradation is induced by peer-to-peer model mixing under heterogeneous task portfolios, rather than by multi-task learning itself. The similarity-aware weights in Eq.~\eqref{eq:softmax} are designed to mitigate the influence of such incompatible updates during aggregation.

This successive aggregation model allows each client to implicitly incorporate multi-hop information without establishing global connections. Let $\boldsymbol{\Psi}_{r,j} =[\psi_{1,r,j},\cdots,\psi_{N,r,j}]^{\top} \in \mathbb{R}^{N\times M}$ denote the collective parameter matrix. We define the global aggregation weight matrix as $\boldsymbol{\mathrm{W}}=\{w_{i,n}\}_{i,n=1}^{N}\in\mathbb{R}^{N\times N}$, where $w_{i,n}=0$ for any $i\notin \mathcal{N}_{n}\cup\{n\}$. By construction via Eq.~\eqref{eq:softmax}, $\boldsymbol{\mathrm{W}}$ is a strictly column-stochastic matrix, satisfying $\mathbf{1}_{N}^{\top}\boldsymbol{\mathrm{W}}=\mathbf{1}_{N}^{\top}$ with $\mathbf{1}_{N}=[1,\dots,1]^{\top}$. This column-stochastic property ensures that the global arithmetic mean of the network parameters is perfectly preserved across any aggregation step $j$, i.e., $\frac{1}{N}\sum_{i=1}^N \psi_{i,r,j} = \frac{1}{N}\sum_{i=1}^N \psi_{i,r,j-1}$.

Assuming a fixed and strongly connected topology, the iterative aggregation obeys
\begin{equation}
    \boldsymbol{\Psi}_{r,j}=\boldsymbol{\mathrm{W}}\boldsymbol{\Psi}_{r,j-1}=(\boldsymbol{\mathrm{W}})^{j}\boldsymbol{\Psi}_{r,0}.
\end{equation}

By the Perron-Frobenius theorem, the irreducible column-stochastic matrix $\boldsymbol{\mathrm{W}}$ possesses a unique positive right stationary eigenvector $\boldsymbol{v}\in\mathbb{R}^{N}$ satisfying $\boldsymbol{\mathrm{W}}\boldsymbol{v}=\boldsymbol{v}$, where $v_{i}>0$ and $\mathbf{1}_{N}^{\top}\boldsymbol{v}=1$. Let $\rho=\max\{|\rho_{\mathrm{W}}|:\rho_{\mathrm{W}}\in\operatorname{spec}(\boldsymbol{\mathrm{W}}),\rho_{\mathrm{W}}\neq 1\}<1$ denote the subdominant spectral radius governing the mixing rate. The powers of $\boldsymbol{\mathrm{W}}$ geometrically converge to a rank-one consensus projector expressed as
\begin{equation}\label{eq:j_inf}
    \lim_{j\to \infty, j\in \mathbb{N}_{+}}(\boldsymbol{\mathrm{W}})^{j}=\boldsymbol{v}\mathbf{1}_{N}^{\top}.
\end{equation}

Under the ``communication-while-aggregation'' protocol, executing $J$ aggregations traversing multi-hop paths exponentially suppresses the topology-dependent variance. However, in heterogeneous multi-task settings, an unbounded $J$ aggressively pulls personalized models toward a weighted global state $\boldsymbol{v}\mathbf{1}_N^{\top}\boldsymbol{\Psi}_{r,0}$, amplifying the structural OCB. This structural tension necessitates the existence of an optimal, finite $J^*$.

\begin{algorithm}[htbp] 
    \caption{Communication-while-Aggregation for Personalized DFL}   
    \label{algorithm:dfl_sc}       
    \begin{algorithmic}[1] 
    \REQUIRE  $\mathcal{N}$, $\mathcal{D}$,  $\boldsymbol{\theta}_{0}$,  $\mathcal{T}, \Lambda, R, E, J, \eta$.
    \ENSURE Trained personalized local models $\boldsymbol{\theta}_{R}$.
    \STATE Clients initialize shared backbone parameters $\boldsymbol{\theta}_{0}$.
    \FOR{each communication round $r = 1, \cdots, R$}
        \FOR{each client $i \in \mathcal{N}$ in parallel}
            \STATE Perform parallel local training to obtain $\theta_{i,r,E}$ by executing $E$ epochs of SGD on $\mathcal{D}_i$ via Eq.~\eqref{eq:local_model_update}.
            \STATE Initialize the local aggregation buffer with state, $\psi_{i,r,0} \leftarrow \theta_{i,r,E}$.
            \FOR{each aggregation step $j = 1, \cdots, J$}
                \STATE Broadcast current local state $\psi_{i,r,j-1}$ and synchronize with all semantically aligned one-hop neighbors $n \in \mathcal{N}_i$.
                \STATE Compute the updated mixing state $\psi_{i,r,j}$ by performing similarity-aware neighborhood aggregation via Eq.~\eqref{eq:aggregation}.
            \ENDFOR
            \STATE Set the local model $\theta_{i,r+1,0} \leftarrow \psi_{i,r,J}$ for the subsequent training round.
        \ENDFOR
    \ENDFOR
    \end{algorithmic} 
\end{algorithm}
\section{Theoretical Formulation and Convergence Analysis}
We study a DSC network under the ``communication-while-aggregation'' protocol. To simultaneously reflect global learning and client-level personalization, we introduce a round-wise composite objective that couples the loss at the round mean iterate with a consensus-error term measuring post-aggregation disagreement. Based on this objective, we establish convergence guarantees in the DSC setting and further characterize how the aggregation depth $J$ affects the steady-state performance.

\subsection{Composite Objective for Global Generalization and Personalization}
After establishing the client-wise personalized objective in Eq.~\eqref{eq:personal_obj}, evaluating whether decentralized training reaches a satisfactory operating point requires jointly accounting for global generalization and client-level personalization. In distributed SC with heterogeneous task portfolios, the minimizer of the global objective in Eq.~\eqref{eq:global_obj} differs from the collection of client-specific minimizers. As a result, optimizing the global loss alone may still leave many clients far from their preferred personalized solutions.

To capture this global-personalization discrepancy within each communication round, we consider the model parameters after $J$ aggregations in round $r$. Let $\psi_{i,r,J}$ denote the model parameter of client $i$, and let $\bar{\psi}_{r,J}=\tfrac{1}{N}\sum_{i\in\mathcal{N}}\psi_{i,r,J}$ be their network average. Because the aggregation weight matrix $\boldsymbol{\mathrm{W}}$ is strictly column-stochastic, the intra-round aggregation rigorously preserves the global arithmetic mean according to
\begin{equation}\label{eq:mean-preserve}
    \bar{\psi}_{r,J}=\bar{\theta}_{r,E},
\end{equation}
which implies that the network average after $J$ topological mixings exactly coincides with the global mean state immediately after the local training phase.

Building on this structural property, the exact global loss parameterized by the dispersed local models can be algebraically decomposed into the loss at the global mean and a personalization-induced variation term, formulated as
\begin{align}\label{eq:global-personal-obj-def}
    \mathcal{L}(\Lambda, \mathcal{D}, \Psi_{r,J})
    &= \mathcal{L}(\Lambda, \mathcal{D}, \bar{\psi}_{r,J})
    \nonumber \\  &\hspace{-5em}+ \frac{1}{N} \sum_{i \in \mathcal{N}}
    \Big[
    \mathcal{L}_{i}(\Lambda_{i}, \mathcal{D}_{i}, \psi_{i,r,J})
    - \mathcal{L}_{i}(\Lambda_{i}, \mathcal{D}_{i}, \bar{\psi}_{r,J})
    \Big],
\end{align}
where $\Psi_{r,J}=\{\psi_{i,r,J}\mid i\in\mathcal{N}\}$. The first term reflects the generalized optimization progress, while the second term captures the exact loss deviation induced by task heterogeneity.

To analytically upper-bound this exact deviation, we invoke the $\mathrm{L_{eff}}$-smoothness property (formally defined in Assumption 1). For any client $i$, the smoothness inequality yields
\begin{align}\label{eq:local-loss-dev-bound}
    \mathcal{L}_{i}\!\left(\Lambda_{i},\mathcal{D}_{i},\psi_{i,r,J}\right)    &\le\mathcal{L}_{i}\!\left(\Lambda_{i},\mathcal{D}_{i},\bar{\psi}_{r,J}\right)+
    \frac{\mathrm{L_{eff}}}{2}\left\|\psi_{i,r,J}-\bar{\psi}_{r,J}\right\|^{2}\nonumber \\
    &\hspace{-1em} +\left\langle\nabla \mathcal{L}_{i}\!\left(\Lambda_{i},\mathcal{D}_{i},\bar{\psi}_{r,J}\right),\psi_{i,r,J}-\bar{\psi}_{r,J}\right\rangle.
\end{align}
When averaging Eq.~\eqref{eq:local-loss-dev-bound} over all clients, the residual inner product term does not trivially vanish due to the non-IID data distributions. However, by applying the Cauchy-Schwarz and Young's inequalities, this inner product can be strictly decoupled and bounded by a combination of the squared gradient norms and the squared parameter dispersion. 

This mathematical decoupling reveals that the exact loss deviation in Eq.~\eqref{eq:global-personal-obj-def} is dominated by the network consensus variance. We define this round-wise consensus error as
\begin{equation}\label{eq:app_Econs_def}
    E_{\mathrm{cons},r,J} \triangleq \frac{1}{N}\sum_{i\in\mathcal{N}} \left\|\psi_{i,r,J}-\bar{\psi}_{r,J}\right\|^{2},
\end{equation}
where $E_{\mathrm{cons},r,J}$ serves as a tractable metric for geometric disagreement across clients. 

Because the exact physical loss $\mathcal{L}(\Lambda, \mathcal{D}, \Psi_{r,J})$ is difficult to optimize directly in a decentralized topology, the aforementioned bounding technique allows us to construct a surrogate upper bound. We formulate the following global- and personalization-oriented composite Lyapunov potential function $\Phi_{r,J}$ given by
\begin{equation}\label{eq:Lyapunov_potential}
    \Phi_{r,J} = \mathcal{L}(\Lambda,\mathcal{D},\bar{\theta}_{r,E}) + \alpha E_{\mathrm{cons},r,J},
\end{equation}
where $\alpha > 0$ is a composite penalty weight incorporating the Lipschitz constant $\mathrm{L_{eff}}$ and Young's inequality scaling factors. Our convergence analysis focuses on the Lyapunov potential $\Phi_{r,J}$, whose minimization ensures descent of the exact loss: it drives the network mean toward a global stationary point while $J$ constrains personalized models within a bounded consensus neighborhood.

\subsection{Convergence Analysis}
We now formally analyze the convergence behavior of the proposed DFL system under heterogeneous multi-task settings. To simplify the subsequent derivations, we introduce the following shorthand notations for the task-wise and client-wise objective functions: $\mathcal{L}_{i,t}(\theta_{i,t}) \triangleq \mathcal{L}_{i,t}(\lambda_{i,t},\mathcal{D}_{i,t},\theta_{i,t})$ and $\mathcal{L}_{i}(\theta_{i}) \triangleq \mathcal{L}_{i}(\Lambda_{i},\mathcal{D}_{i},\theta_{i})$.

\newtheorem{Definition}{Definition}
\begin{Definition}[Statistical Heterogeneity]\label{def:stat_hetero}
At a common reference parameter vector $\theta$, let $\nabla \mathcal{L}_{i}(\theta)$ denote the exact gradient of the local objective of client $i$. Define the mean gradient over all clients as
\begin{equation}\label{eq:stat_def}
    \bar{g}(\theta)\triangleq \frac{1}{N}\sum_{i\in\mathcal{N}} \nabla\mathcal{L}_{i}(\theta).
\end{equation}
The \emph{statistical heterogeneity} at $\theta$ is quantified by the gradient dispersion formulated as
\begin{equation}\label{eq:Gamma_def}
    \Gamma(\theta) \triangleq \frac{1}{N}\sum_{i\in\mathcal{N}} \big\|\nabla\mathcal{L}_{i}(\theta) - \bar{g}(\theta)\big\|^{2}.
\end{equation}
\end{Definition}

\begin{Definition}[Structural Heterogeneity and OCB]\label{def:struct_hetero_ocb}
Let $\theta_{i}^{*}$ be a minimizer of the local objective for client $i$, and define the global average of these client-specific optima as
\begin{equation}\label{eq:theta_star_bar_def}
    \bar{\theta}^{*} \triangleq \frac{1}{N}\sum_{i\in\mathcal{N}} \theta_{i}^{*}.
\end{equation}
The \emph{structural heterogeneity} is measured by the geometric dispersion of these client optima given by
\begin{equation}\label{eq:Omega_def}
    \Omega_{*}^{2} \triangleq \frac{1}{N}\sum_{i\in\mathcal{N}} \big\|\theta_{i}^{*} - \bar{\theta}^{*}\big\|^{2}.
\end{equation}
We refer to $\Omega_*^2$ as the OCB. When $\Omega_*^2>0$, a single consensus model cannot simultaneously match all client optima. From the aggregation perspective, $\Omega_*^2$ measures the dispersion of client-specific personalized optima around a common consensus point. When $\Omega_*^2$ is small, the personalized optima are relatively close, and additional aggregation is less likely to introduce a large structural bias. A larger $\Omega_*^2$ reflects stronger disagreement among local optima. In that case, deeper mixing can reduce consensus variance, but it may also move local models away from their preferred personalized solutions. This persistent deviation corresponds to the OCB term in our convergence analysis.

\end{Definition}

\begin{assumption}[Lipschitz Smoothness]\label{assumption:smoothness}
For client $i$ and local task $t$, the task-specific loss function is $\mathrm{L}_{i,t}$-Lipschitz smooth. Specifically, for any two model parameters $\theta_{i,t}$ and $\theta_{i,t}^{'}$, it holds that
\begin{equation}\label{eq:L-smooth}
    \|\nabla\mathcal{L}_{i,t}(\theta_{i,t}) - \nabla\mathcal{L}_{i,t}(\theta_{i,t}^{'})\| \leq \mathrm{L}_{i,t}\,\|\theta_{i,t} - \theta_{i,t}^{'}\|.
\end{equation}
The effective smoothness constant across all clients and tasks is defined as $\mathrm{L}_{\mathrm{eff}} \triangleq \max_{i}\sum_{t\in \mathcal{T}_{i}}\mathrm{L}_{i,t}$.
\end{assumption}

\begin{assumption}[Polyak-Łojasiewicz (PL) Condition] \label{assumption: PL}
The aggregate local objective $\mathcal{L}_i$ for each client $i$, as well as the global objective $\mathcal{L}$, satisfy the PL condition~\cite{PL_condition}. Specifically, there exists an effective PL constant $\mu_{\mathrm{eff}} > 0$ such that for any parameter $\theta$
\begin{equation}
    2\mu_{\mathrm{eff}}\big(\mathcal{L}(\theta) - \mathcal{L}^{*}\big) \leq \|\nabla \mathcal{L}(\theta)\|^{2},
\end{equation}
where $\mathcal{L}^{*}$ denotes the optimal value of the corresponding loss function. This condition is used as a regularity assumption for the convergence argument, rather than as a global description of the full neural objective.  
\end{assumption}

\begin{remark}[Role of PL condition] The PL condition is introduced as a regularity condition for the convergence argument, rather than as a claim that the full neural objective is globally well behaved over the entire parameter space. PL-type inequalities are commonly used to establish linear convergence for non-convex objectives without requiring convexity~\cite{PL_condition}, and related analyses have also discussed such behaviors in over-parameterized nonlinear models~\cite{Liu2021Loss}. In our analysis, this condition is only invoked when converting the Lyapunov drift relation into the linear objective-gap recursion in Theorem~\ref{Theorem: linear}. The preceding drift bound does not rely on this conversion; without the PL step, it can still be interpreted as a stationarity-oriented result in terms of the gradient norm and consensus error. \end{remark}

\begin{assumption}[Unbiased Stochastic Gradients and Bounded Variance]\label{assumption: unbiased gradient}
For each client $i$, the stochastic gradient $g_{i,r,e}$ at local step $e$ of round $r$ is an unbiased estimator of the true local gradient, satisfying
\begin{equation}
    \mathbb{E}\big[g_{i,r,e}\,\big|\,\theta_{i,r,e}\big] = \nabla \mathcal{L}_{i}(\theta_{i,r,e}).
\end{equation}
Furthermore, its conditional variance is uniformly bounded as
\begin{equation}\label{eq:variance}
    \mathbb{E}\big[\|g_{i,r,e}-\nabla\mathcal{L}_{i}(\theta_{i,r,e})\|^{2}\,\big|\,\theta_{i,r,e}\big] \leq \sigma_{i}^{2},
\end{equation}
where $\sigma_{i}^{2} \ge 0$ represents the local gradient variance bound for client $i$, and $\sigma^{2}=\tfrac{1}{N}\sum_{i\in\mathcal{N}}\sigma_{i}^{2}$ denotes the network-wide average gradient noise level.
\end{assumption}

\begin{assumption}[Stepsize]\label{assumption: stepsize}
The per-round learning rate is chosen as $\eta = \kappa\cdot\frac{1-\rho^{2J}}{\mathrm{L}_{\mathrm{eff}}E}$, where $0<\kappa \le \kappa_{\max} = \min\!\big\{\tfrac{1}{4},\,\tfrac{\mu_{\mathrm{eff}}}{4 \mathrm{L}_{\mathrm{eff}}}\big\}$.
\end{assumption}

\begin{assumption}[Bounded Statistical Heterogeneity]\label{assump:bounded_stat_hetero}
The statistical heterogeneity measure defined in Eq.~\eqref{eq:Gamma_def} is uniformly bounded. There exists such a constant $\widehat{\Gamma}\ge 0$ that $\Gamma(\theta)\le \widehat{\Gamma}$ for all $\theta$.
\end{assumption}
\begin{remark}[Statistical and structural heterogeneity] The quantity $\Gamma(\theta)$ measures the dispersion of local gradients evaluated at a common parameter $\theta$, and therefore characterizes the statistical mismatch among client objectives during local optimization. Similar bounded dissimilarity or bounded gradient-dissimilarity conditions are widely used in heterogeneous FL analyses to make use of non-IID data analytically tractable~\cite{Li2020Federated, Wang2022Unreasonable}. In the proposed DSC setting, this gradient-level mismatch arises from both heterogeneous data distributions and heterogeneous task portfolios. Since the clients optimize over a shared semantic task space through a common multi-path encoder, $\widehat{\Gamma}$ provides an effective upper bound on the gradient dispersion along the considered training trajectory. This statistical heterogeneity is distinct from the structural heterogeneity measured by $\Omega_*^2$, which describes the dispersion of client-specific optima and gives rise to the OCB term. Consequently, $\widehat{\Gamma}$ affects the noise- and variance-related constants in the convergence bound, whereas $\Omega_*^2$ determines the bias induced by excessive aggregation depth. \end{remark}

The coupled dynamics of local training and topology-aware aggregation govern the convergence trajectory through two competing mechanisms. On one hand, local SGD minimizes the objective but inevitably injects structural disagreement due to data heterogeneity and gradient noise. On the other hand, the iterative gossip protocol contracts this disagreement, though overly aggressive mixing may trigger an irreducible structural bias. To rigorously decouple these effects, we first establish a cross-round recursion for the consensus error.

\begin{lemma}[Cross-Round Recursion of Consensus Error]\label{lemma: grad_and_disagree_bound}
Under Assumptions~\ref{assumption:smoothness}--\ref{assumption: stepsize}, there exist universal constants $a_1>0$, $a_2>0$, and $c_{\mathrm{OCB}}>0$ such that for all $r\ge 0$ the consensus error is bounded by
\begin{align}\label{eq:cross-round-recursion}
    E_{\mathrm{cons},r+1,J}
    &\le \rho^{2J}\Big(1+a_1\eta^{2}E^{2}\mathrm{L}_{\mathrm{eff}}^{2}\Big)E_{\mathrm{cons},r,J} \nonumber\\
    &\hspace{-2em} +\rho^{2J}a_2\eta^{2}E^{2}\big(\widehat{\Gamma}+\sigma^{2}\big) +c_{\mathrm{OCB}}(1-\rho^{2J})\Omega_{*}^{2}.
\end{align}
\end{lemma}
\begin{IEEEproof}
See Appendix~\ref{Proof: grad_and_disagree_bound}.
\end{IEEEproof}

In the recursion bound presented above, the first term on the right-hand side (RHS) of Eq.~\eqref{eq:cross-round-recursion} characterizes the expansion of disagreement due to $E$ epochs, followed by the contraction induced by $J$ gossip steps. The second term  on the RHS captures the divergence injected by gradient noise and statistical heterogeneity, while the final term mathematically isolates the irreducible OCB. This explicit term-by-term separation forms the quantitative foundation for analyzing how the aggregation depth $J$ shapes network coherence. By substituting this disagreement evolution into a smoothness-based descent inequality, we derive a unified Lyapunov drift bound.

\begin{lemma}[Per-Round Lyapunov Drift Bound]\label{lemma: decrease of potential}
We define the expected per-round drift as $\Delta \Phi_r \triangleq \mathbb{E}[\Phi_{r+1,J} - \Phi_{r,J}]$. Under Assumptions~\ref{assumption:smoothness}--\ref{assumption: stepsize}, there exist universal constants $c_{\Phi,1}, c_{\Phi,3} \in (0,1)$ and $c_{\Phi,2}, C_{\mathcal{L}} > 0$ such that if the personalization weight $\alpha$ satisfies the admissible condition given by
\begin{equation}\label{eq:alpha range}
    \alpha\Big((1-\rho^{2J})-\rho^{2J}a_1\eta^{2}E^{2}\mathrm{L}_{\mathrm{eff}}^{2}\Big) \ge (2 C_{\mathcal{L}} + c_{\Phi,3})\,\eta E\,\mathrm{L}_{\mathrm{eff}}^{2},
\end{equation}
then the drift is strictly bounded by
\begin{align}\label{eq:drift-bound}
    \Delta \Phi_r &\le -c_{\Phi,1}\eta E\,\mathbb{E}\big\|\nabla\mathcal{L}(\bar{\theta}_{r,E})\big\|^{2} -c_{\Phi,3}(1-\rho^{2J})\,\mathbb{E}[E_{\mathrm{cons},r,J}] \nonumber \\
    &\quad +c_{\Phi,2}\eta E(\widehat{\Gamma}+\sigma^{2}) +\alpha c_{\mathrm{OCB}}(1-\rho^{2J})\Omega_{*}^{2}.
\end{align}
\end{lemma}
\begin{IEEEproof}
See Appendix~\ref{Proof: decrease of potential}.
\end{IEEEproof}

Lemma~\ref{lemma: decrease of potential} successfully couples global optimization progress with consensus error contraction. By invoking the PL condition on this drift, we can formally extract the linear convergence rate and the exact $J$-dependent steady-state error floor.

\begin{Theorem}[Linear Convergence and Trade-off Characterization]\label{Theorem: linear}
Under Assumptions~\ref{assumption:smoothness}--\ref{assump:bounded_stat_hetero}, and for any $\alpha>0$ satisfying Eq.~\eqref{eq:alpha range}, the composite objective $\Phi_{r,J}$ converges linearly to a steady-state error floor, as given by
\begin{align}\label{eq:linear-bound}
    \mathbb{E}\Big[\Phi_{r,J}-\mathcal{L}^{*}\Big] &\le (1-\kappa_{\mathrm{lin}})^{r}\Big(\Phi_{0,J}-\mathcal{L}^{*}\Big) + C_{\text{base}}\nonumber \\
    &\quad + \frac{C_{\Gamma}}{1-\rho^{2J}} + \alpha C_{\Omega}(1-\rho^{2J})\Omega_{*}^{2},
\end{align}
where $\kappa_{\mathrm{lin}} = c_{\Phi,1}\mu_{\mathrm{eff}}\eta E$ defines the linear contraction rate. The geometric bounding constants evaluate to $C_{\text{base}} = \frac{\eta c_{\Phi,2}}{c_{\Phi,1}\mu_{\mathrm{eff}}}(\widehat{\Gamma}+\sigma^{2})$, $C_{\Gamma} = \frac{\eta^2 c_{\text{var}}}{c_{\Phi,1}\mu_{\mathrm{eff}}}(\widehat{\Gamma}+\sigma^{2})$, and $C_{\Omega} = \frac{c_{\mathrm{OCB}}}{c_{\Phi,1}\mu_{\mathrm{eff}}}$. All three terms remain strictly positive. Structurally, they are entirely decoupled from both the communication round index $r$ and the aggregation depth $J$.
\end{Theorem}
\begin{IEEEproof}
See Appendix~\ref{Proof: Theorem linear}.
\end{IEEEproof}

Theorem~\ref{Theorem: linear} exposes a mechanical conflict driven by the aggregation depth $J$. Pushing communication deeper suppresses the topology-dependent consensus variance. However, it simultaneously inflates the structural OCB. This U-shaped error floor forces a critical design choice: resource-constrained edge deployments must lock into an optimal depth $J^*$ to prevent model divergence.

\subsection{Optimal Aggregation Depth via Upper Bound Minimization}\label{subsec:optJ}

To resolve this structural tension, we must isolate the terms actively controlled by $J$ inside the theoretical upper bound. High task heterogeneity violently scatters the local optima across the parameter space. This exact scattering dictates the severe OCB penalty during deep mixing. We therefore extract $J$ as our primary decision variable.

System hardware and bandwidth budgets inherently restrict $J$ to a discrete feasible domain:
\begin{equation}\label{eq:J_domain}
    J \in \{1, 2, \ldots, J_{\max}\},
\end{equation}
with $J_{\max} \in \mathbb{N}_{+}$ acting as the hard communication ceiling. Next, we discard the intrinsic SGD noise $C_{\text{base}}$ from the error bound. It remains entirely invariant to the mixing steps. By defining $A \triangleq C_{\Gamma}$ to capture the topological variance scale and $B \triangleq C_{\Omega}\Omega_*^2$ to quantify the structural bias, we condense the $J$-dependent dynamics into a simplified objective function:
\begin{equation}\label{eq:BoundJ_abstract}
    \mathrm{Bound}(J) \triangleq \frac{A}{1-\rho^{2J}} + \alpha B(1-\rho^{2J}).
\end{equation}
Locating the optimal operating point $J^{*}$ thus reduces to a constrained discrete minimization problem:
\begin{equation}\label{eq:PJ}
    J^{*} \in \arg\min_{J \in \{1,\dots,J_{\max}\}} \mathrm{Bound}(J),
\end{equation}
which must strictly satisfy the admissibility limits established earlier in Lemma~\ref{lemma: decrease of potential}.

Eq.~\eqref{eq:BoundJ_abstract} translates our U-shaped trade-off into pure algebra. Expanding $J$ pushes the effective mixing gap $(1-\rho^{2J})$ toward unity. The variance fraction $A/(1-\rho^{2J})$ decays rapidly, but the structural bias term $\alpha B(1-\rho^{2J})$ grows right alongside it. In highly heterogeneous environments where $\Omega^{2}_{*} \gg 0$, increasing the gossip depth helps at first by crushing gradient noise. However, performance soon collapses once OCB takes over the error floor. The opposite happens in strictly homogeneous settings where $\Omega_{*}^{2} \to 0$. Here, the bias penalty disappears entirely, meaning the network always benefits from a larger $J$.

\begin{Theorem}[Optimal Aggregation Depth $J^{*}$]\label{Theorem: Jstar}
Let us relax the steady-state objective in Eq.~\eqref{eq:BoundJ_abstract} to the continuous domain. If structural heterogeneity is strong enough to satisfy the interior condition $\alpha B > A$, the unique global minimizer is given by
\begin{equation}\label{eq:Jstar}
    J^{*} = \frac{1}{-2\ln\rho}\ln\left(\frac{1}{1-\sqrt{\frac{A}{\alpha B}}}\right).
\end{equation}
Otherwise, if the network is highly homogeneous such that $\alpha B \le A$, the function $\mathrm{Bound}(J)$ decreases monotonically. In this scenario, the theoretical optimal aggregation depth has no upper bound.
\end{Theorem}
\begin{IEEEproof}
See Appendix~\ref{Proof: Theorem Jstar}.
\end{IEEEproof}

Theorem~\ref{Theorem: Jstar} provides a direct configuration rule for system deployment. When $\alpha B > A$, Eq.~\eqref{eq:Jstar} pinpoints the exact boundary where variance reduction perfectly balances structural bias. Because practical communication depths require integer values, the actual deployed $J$ is obtained by testing the two nearest integers to the continuous root $J^{*}$. The final value is strictly capped by the hardware threshold $J_{\max}$. This explicit configuration maps the theoretical minimum directly to real-world edge bandwidth constraints.

\section{Experiments}
This section evaluates the similarity-aware aggregation protocol against standard peer-to-peer baselines. The primary goal is to verify the multi-task performance gains within heavily personalized client networks. We sweep the aggregation depth $J$ to observe the convergence behavior. These empirical observations directly validate the theoretical optimal mixing depth derived in Section~\ref{subsec:optJ}. Finally, we examine whether the observed gains remain stable under network-size variation and how imperfect wireless links affect packet loss and communication latency.

\subsection{Experimental Setup}
\textbf{Datasets and task settings.}
We conduct experiments on both NYU-v2 and Taskonomy. The NYU-v2 benchmark is used as the main testbed, where we evaluate the effectiveness of the proposed similarity-aware aggregation, examine the convergence behavior, and study the influence of the aggregation depth $J$. NYU-v2 contains 795 training and 654 test images with dense annotations for semantic segmentation (Seg), depth estimation (Dep), and surface normal prediction (Sn). Unless otherwise specified, the simulated DSC network on NYU-v2 consists of $N=10$ clients.

To examine whether the observed gains remain stable beyond this setting, we also conduct supplementary experiments on Taskonomy. This dataset introduces a different multi-task composition, including Seg, Sn, keypoint detection (KP), and edge detection (Edge). In the Taskonomy experiments, we vary the number of clients over $N\in [2,30]$. 

\textbf{Communication setup.}
We construct the decentralized topology from randomly generated client locations in a two-dimensional area. Rather than assuming a fully connected mesh, we use a random geometric graph to reflect distance-constrained peer-to-peer communication. Two clients are connected when their distance falls within the communication radius. For each topology realization, the radius is gradually increased until the graph becomes connected, and the resulting adjacency defines the one-hop neighbor set $\mathcal{N}_i$ used for decentralized aggregation.

For wireless access, we adopt the TDMA-based coloring schedule described in the system model. Feasible one-hop links are assigned to time slots such that conflicting transmissions are separated. This setting provides a controlled abstraction of local wireless contention, while keeping the comparison focused on how different aggregation rules exploit the same neighborhood structure.

We further incorporate packet-level reliability into the communication evaluation. For communication-volume accounting, both the shared encoder parameters and the routing-policy entries are represented using 32-bit floating-point values. Each transmitted message is divided into fixed-length packets, and the packet number on each link is computed from the message size. The packet loss probability of a valid one-hop link depends on the nominal regional SNR, the inter-client distance, and the packet length. For latency evaluation, links assigned to the same TDMA color slot are treated as parallel transmissions, and the slot duration is determined by the maximum link latency among the active links in that slot. The latency of one aggregation step is obtained by summing the durations of all color slots, while the total communication latency is accumulated over the required communication rounds and aggregation steps. In the communication-quality experiments, both packet loss probability and communication latency are averaged over random topology realizations.

\textbf{Implementation Details.}
The local models utilize a DeepLab-ResNet-34 backbone. For a given client $i$, the backbone splits into $|\mathcal{T}_{i}|$ task-specific decoders. We train the network weights using standard SGD and update the routing policy parameters via the Adam optimizer. Both optimizers are initialized with a learning rate of $1 \times 10^{-4}$. The training process spans $R=60$ global communication rounds. During each round, clients complete $E=10$ local epochs before triggering the gossip aggregation. The aggregation depth $J$ is evaluated over the interval $[1,15]$.

\textbf{Evaluation Metrics.}
We track the global loss convergence alongside the final task-specific performance. Each downstream task relies on its standard evaluation metrics. Seg is evaluated via mean intersection over union (mIoU) and pixel accuracy (PAcc). Sn performance is measured using mean and median angular errors, plus the accuracy thresholds at $11.25^{\circ}$, $22.5^{\circ}$, and $30^{\circ}$. For Dep, we calculate absolute (Abs) and relative (Rel) errors, along with the threshold percentage $\delta \in \{1.25, 1.25^{2}, 1.25^{3}\}$. For KP and Edge, we calculate the $L_1$ loss, where lower values indicate better performance.

We establish a strict \textit{no-aggregation} baseline to isolate the collaborative gains. Let $b_{t,q}$ denote this baseline for metric $q$ on task $t$. If client $i$ achieves an experimental measurement $m_{i,t,q}$, its relative improvement is calculated as $\Delta_{i,t,q} = (m_{i,t,q} - b_{t,q})/b_{t,q} \times 100\%$. The evaluation metrics scale in opposite directions; higher is better for accuracy, while lower is better for errors. A directional multiplier $I_{t,q} \in \{-1, 1\}$ is therefore introduced to unify the metric orientation. The final overall improvement is computed through the following hierarchical aggregation:
\begin{align}\label{eq:delta_hierarchy}
   & \Delta_{i,t} = \frac{1}{|Q_{t}|} \sum_{q \in Q_{t}} I_{t,q} \Delta_{i,t,q}, \\& \Delta_{t}= \frac{1}{|\mathcal{N}_{t}|} \sum_{i \in \mathcal{N}_{t}} \Delta_{i,t}, \\
   & \Delta_{\mathrm{all}} = \textstyle \frac{1}{|\mathcal{T}|} \sum_{t \in \mathcal{T}} \Delta_{t},   
\end{align}
Here, $Q_t$ is the specific metric set for task $t$. The subset of clients assigned to task $t$ is denoted by $\mathcal{N}_t$, and $\mathcal{T}$ represents the comprehensive set of all tasks. The relative improvement scores also reflect the effect of negative transfer during aggregation. Since the non-aggregation model is used as the no-transfer reference, $\Delta_{i,t}<0$ indicates that aggregation degrades task $t$ on client $i$. Here, $\Delta_t$ measures the average transfer effect for a specific task; $\Delta_{\rm all}$ reports the network-level average across all tasks. A positive $\Delta_{\rm all}$ means that aggregation is beneficial on average, although some individual clients or tasks may suffer from negative transfer.

\subsection{Results and Analysis}
Most existing PFL frameworks assume single-task setups or rely on central orchestration. They fail in decentralized MT-SC environments. We benchmark our framework against four standard peer-to-peer alternatives to highlight the necessity of task-aware routing:

\begin{enumerate}
    \item \textbf{No-aggregation:} The strict performance lower bound. Clients train in total isolation.
    \item \textbf{Decentralized FedAvg}~\cite{8950073}\textbf{:} A topology-agnostic baseline that weights all neighbors uniformly, ignoring task differences entirely.
    \item \textbf{Heuristic \textit{max} aggregation:}  A naive task-aware approach. Rather than evaluating the full portfolio, it defines peer weights using only the single highest task similarity found between two clients.
    \item \textbf{FedAMP~\cite{Huang2021Personalized}:} A PFL method based on attentive message passing. It measures the similarity between client models and gives higher weights to more related clients during aggregation.
\end{enumerate}

\begin{table*}[h]\footnotesize
\centering
\caption{Comparison of Different Aggregation Methods.}
\label{tab:multi-task-performance-pro}
\begin{tabular}{llllllllllllllllll }
\toprule
\multirow{2}{*}{Method} & 
\multicolumn{3}{c}{\textbf{Seg}} & 
\multicolumn{6}{c}{\textbf{Sn}} & 
\multicolumn{6}{c}{\textbf{Dep}} & 
\multirow{2}{*}{$\Delta_{\mathrm{all}}$ (\%)} \\
\cmidrule(lr){2-4} \cmidrule(lr){5-10} \cmidrule(lr){11-16}
& 
{mIoU} & {PAcc} & {$\Delta_{\mathrm{Seg}}$} & 
{Mean $\downarrow$} & {Median $\downarrow$} & {$11.25^\circ$} & {$22.5^\circ$} & {$30^\circ$} & {$\Delta_{\mathrm{Sn}}$} & 
{Abs$\downarrow$} & {Rel$\downarrow$} & {1.25} & {1.25$^2$} & {1.25$^3$} & {$\Delta_{\mathrm{Dep}}$} & \\
\midrule
No aggregation   & 28.46 & 59.21 & - & 16.54 & 13.05 & 44.11 & 73.12 & 83.14 & - & 0.565 & 0.223 & 62.38 & 88.72 & 96.71  & -    & -  \\
FedAvg($J$=1)      & 28.89 & 59.72 & 1.19 & 16.29 & 12.83 & 44.80 & 73.62 & 83.66 & 1.21 & 0.557 & 0.219 & 63.30 & 89.16 & 97.14 & 1.12  & 1.17  \\
max($J$=1)         & 29.31 & 60.31 & 2.42 & 16.30 & 12.66 & 45.36 & 73.78 & 84.60 & 1.99 & 0.570 & 0.220 & 61.44 & 88.81 & 97.06 & $-$0.12 & 1.43  \\
 FedAMP($J$=1)      & 29.14 & 59.77 & 1.67 & 16.23 & 12.74 & 45.16 & 73.80 & 84.67 & 1.88 & 0.553 & 0.215 & 63.42 & 89.48 & 97.32 & 1.77 & 1.77   \\
Similarity($J$=1)  & 29.26 & 59.92 & 2.01 & 16.22 & 12.67 & 45.33 & 73.89 & 84.84 & 2.14 & 0.551 & 0.212 & 63.49 & 89.67 & 97.37 & 2.19 & 2.11 \\
FedAvg($J$=5)      & 29.37 & 60.87 & 3.00 & 16.18 & 12.68 & 45.25 & 74.50 & 84.85 & 2.31 & 0.551 & 0.217 & 63.94 & 90.22 & 97.60 & 2.06 & 2.46 \\
max($J$=5)      & \textbf{31.87} & \textbf{64.24} & \textbf{10.24} & 16.16 & 12.77 & 45.06 & 74.02 & 84.90 & 1.99 & 0.572 & 0.226 & 61.63 & 87.66 & 95.55 & $-1.24$ & 3.66 \\
 FedAMP($J$=5)     & 31.21 & 61.74 & 6.97 & 16.18 & 12.57 & 45.48 & 74.34 & 84.91 & 2.55 & 0.546 & 0.213 & 63.82 & 90.26 & 97.63 & 2.57 & 4.03   \\
Similarity($J$=5)  & 31.43 & 62.55 & 8.04 & \textbf{16.15} & \textbf{12.51} & \textbf{45.85} & \textbf{74.69} & \textbf{84.93} & \textbf{2.95} & \textbf{0.541} & \textbf{0.209} & \textbf{64.26} & \textbf{90.49} & \textbf{97.68} & \textbf{3.31} & \textbf{4.77} \\
\bottomrule
\end{tabular}
\end{table*}

\begin{figure}[t]
    \centering
    \begin{subfigure}[b]{0.4\linewidth}
        \centering
        \includegraphics[width=\linewidth]{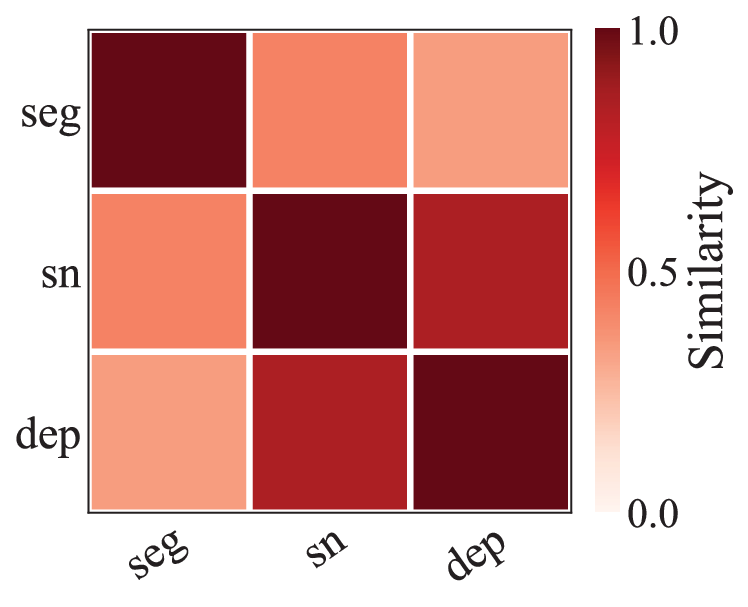}
        \caption{Single task similarity.}
        \label{fig:task_similarity_heatmap}
    \end{subfigure}
    \hfill
    \begin{subfigure}[b]{0.48\linewidth}
        \centering
        \includegraphics[width=\linewidth]{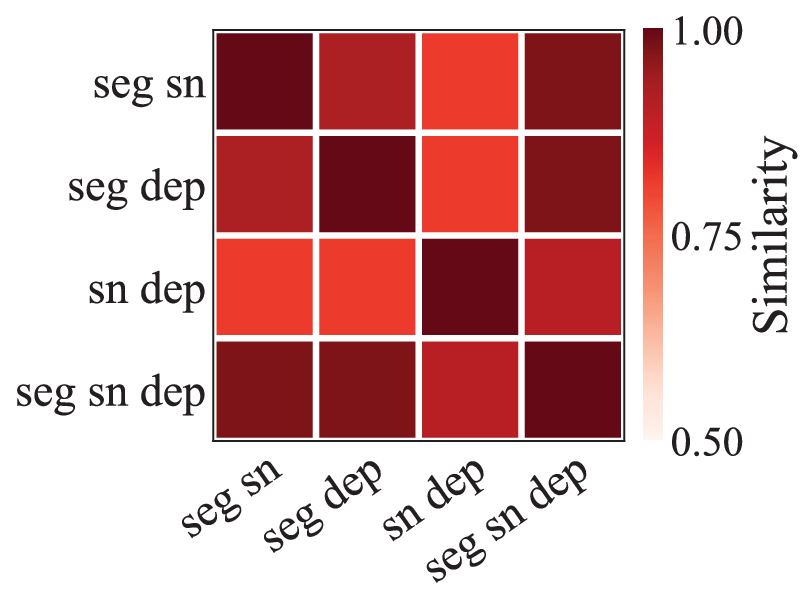}
        \caption{Task combination similarity.}
        \label{fig:task_combo_similarity_heatmap}
    \end{subfigure}
    \caption{Task Correlation Heatmap.}
    \label{fig:task_correlation_heatmap}
\end{figure}

Table~\ref{tab:multi-task-performance-pro} reports the relative gains against the isolated baseline for both shallow ($J=1$) and deep ($J=5$) mixing configurations. To contextualize these results, we must reference the task affinities in Fig.~\ref{fig:task_correlation_heatmap}. The matrix confirms a strong semantic correlation linking Sn with Dep. The stronger correlation between Sn and Dep is consistent with prior visual task-transfer studies~\cite{Zamir}. Seg sits apart, showing weak affinity to both. This specific structural asymmetry dictates the severity of the OCB penalty. As the network mixes deeper, aggregation rules that do not explicitly model task-path compatibility are more likely to introduce negative transfer. Our similarity-aware design inherently blocks this, distributing collaborative gains evenly without triggering sudden performance collapses.

\textbf{Shallow Aggregation.}
Under limited mixing, Decentralized FedAvg delivers only a marginal overall improvement of $1.17\%$. Uniform neighbor weighting indiscriminately injects mismatched semantic features, which severely dilutes the potential collaborative gain. The greedy max strategy improves Seg by $2.42\%$ and Sn by $1.99\%$, but does so at the direct expense of Dep, whose performance drops by $0.12\%$. This result confirms that Dep is particularly sensitive to representation drift. Absorbing parameters from an uncorrelated task such as Seg actively degrades its geometric structure. FedAMP improves the overall gain to $1.77\%$, showing that model-level attentive weighting can reduce part of the mismatch among clients. Its gain is still lower than that of our method, since the attention weights are not explicitly tied to task-path compatibility. In contrast, our proposed method avoids this failure mode and achieves a more balanced overall improvement of $2.11\%$.

\textbf{Deep Aggregation and Cross-Task Interference.}
Increasing the communication depth to $J=5$ amplifies both the variance reduction effect and the underlying structural bias. FedAvg reaches only a modest overall gain of $2.46\%$. Meanwhile, the max strategy induces severe cross-task interference: it pushes Seg to a peak improvement of $10.24\%$, but completely destabilizes Dep, whose performance falls by $1.24\%$. This sharp degradation highlights a critical flaw: propagating incompatible Seg features into Dep models over multiple communication hops destroys their geometric consistency. FedAMP outperforms these two baselines, reaching a $4.03\%$ overall improvement. Its attentive weighting gives related clients more influence and therefore avoids the sharp Dep degradation observed in the max strategy. Nevertheless, the improvement remains below that of the proposed method, especially on Sn and Dep. Our routing scheme instead exploits the known Sn--Dep synergy to filter out these harmful updates. As a result, both tasks remain protected from negative transfer, with Sn improving by $2.95\%$ and Dep by $3.31\%$. Although the Seg gain of our method, $ 8.04\%$, is slightly below that of the greedy baseline, it avoids the over-consensus trap and creates a substantially more stable collaborative environment, ultimately achieving the highest global improvement of $4.77\%$.

\textbf{Convergence Behavior Under Aggregation.}
Fig.~\ref{fig:convergence_all} tracks the training metrics over 60 communication rounds under a deep mixing regime. The primary goal is to confirm optimization stability despite extreme task heterogeneity across the mesh network.

\begin{figure*}[t]
    \centering
    \begin{subfigure}[t]{0.33\linewidth}
        \centering
        \includegraphics[width=\linewidth]{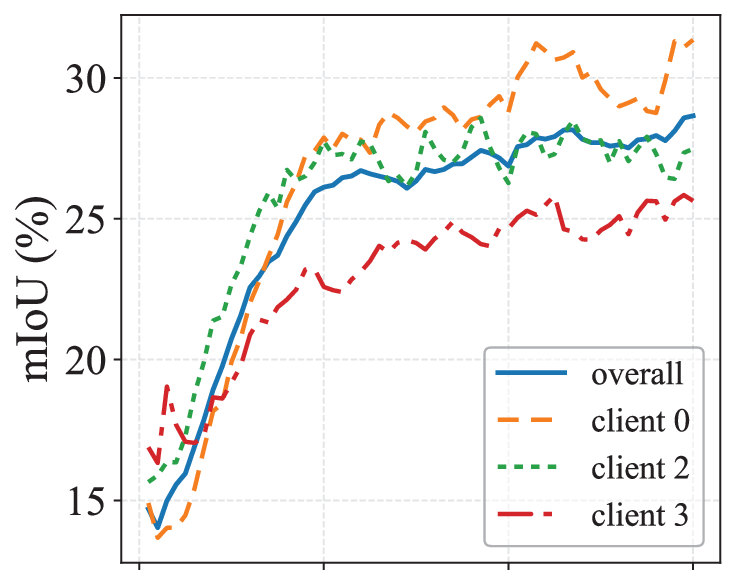}
        \caption{mIoU vs.\ Round (Seg)}
        \label{fig:miou_vs_round}
    \end{subfigure}\hfill
    \begin{subfigure}[t]{0.33\linewidth}
        \centering
        \includegraphics[width=\linewidth]{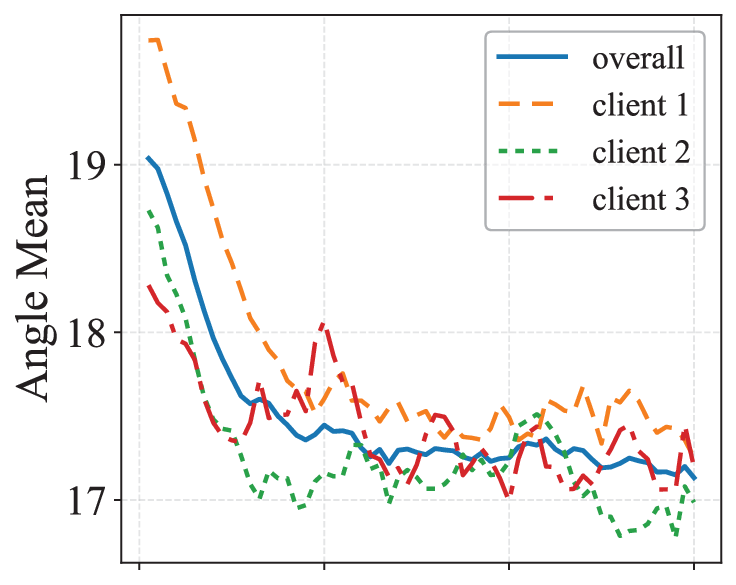}
        \caption{Mean vs.\ Round (Sn)}
        \label{fig:mean_vs_round}
    \end{subfigure}\hfill
    \begin{subfigure}[t]{0.33\linewidth}
        \centering
        \includegraphics[width=\linewidth]{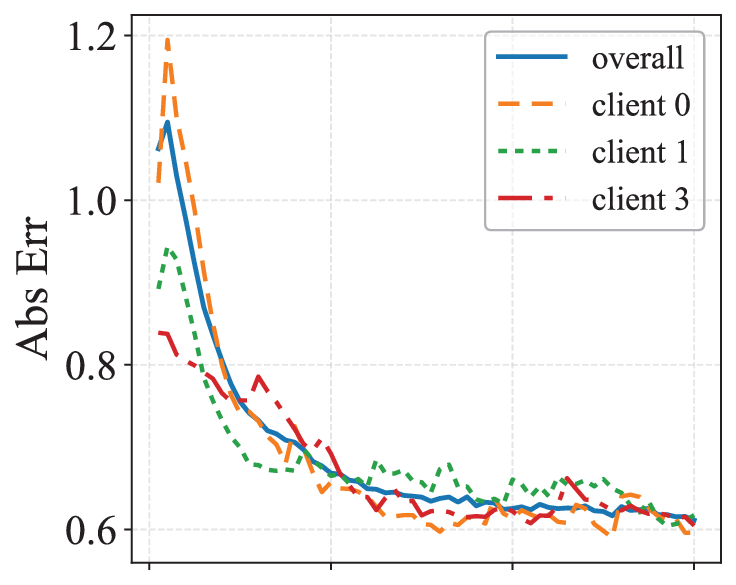}
        \caption{Abs vs.\ Round (Dep)}
        \label{fig:abs_vs_round}
    \end{subfigure}

    \vspace{0cm} 

    \begin{subfigure}[t]{0.33\linewidth}
        \centering
        \includegraphics[width=\linewidth]{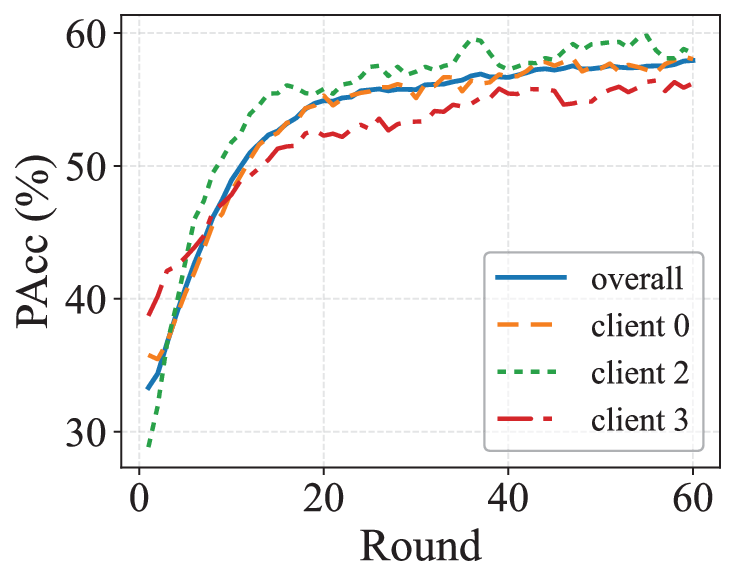}
        \caption{PAcc vs.\ Round (Seg)}
        \label{fig:pacc_vs_round}
    \end{subfigure}\hfill
    \begin{subfigure}[t]{0.33\linewidth}
        \centering
        \includegraphics[width=\linewidth]{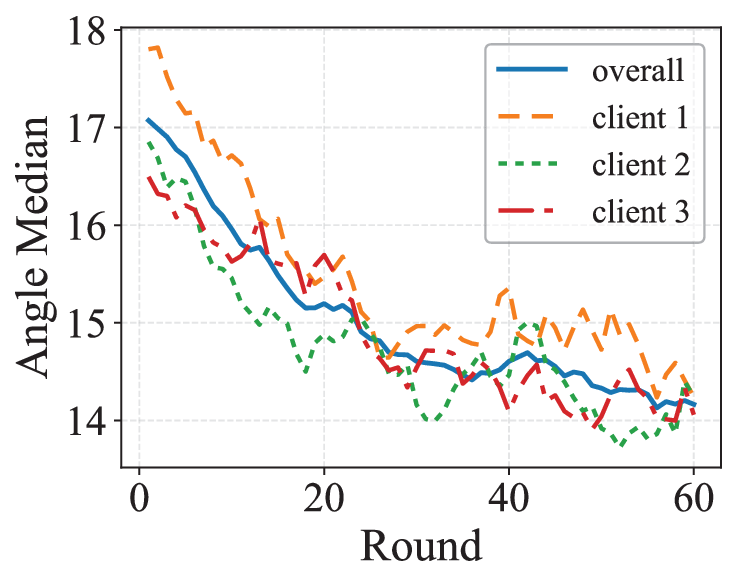}
        \caption{Median vs.\ Round (Sn)}
        \label{fig:median_vs_round}
    \end{subfigure}\hfill
    \begin{subfigure}[t]{0.33\linewidth}
        \centering
        \includegraphics[width=\linewidth]{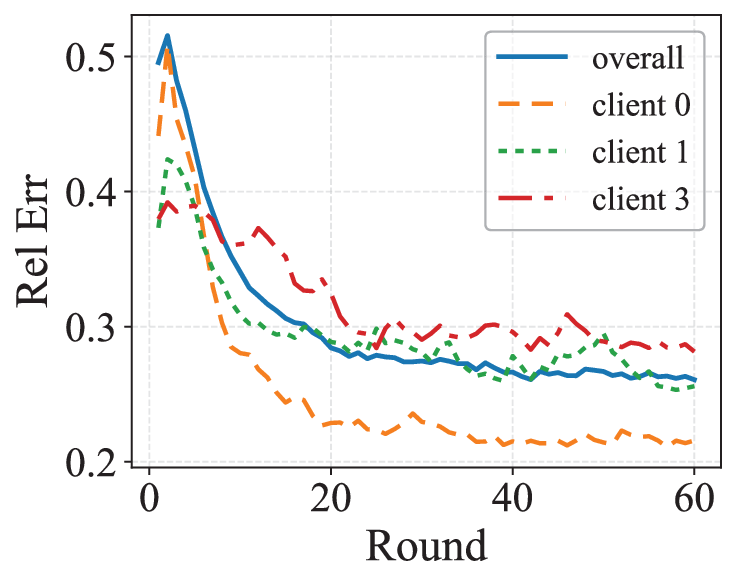}
        \caption{Rel vs.\ Round (Dep)}
        \label{fig:rel_vs_round}
    \end{subfigure}

    \vspace{-0.1cm} 
    \caption{Convergence behavior of different tasks (Seg, Sn, Dep) over communication rounds.}
    \label{fig:convergence_all}
\end{figure*}

\begin{figure}[t] 
    \centering
    \begin{subfigure}[t]{0.48\columnwidth} 
        \centering
        \includegraphics[width=\linewidth]{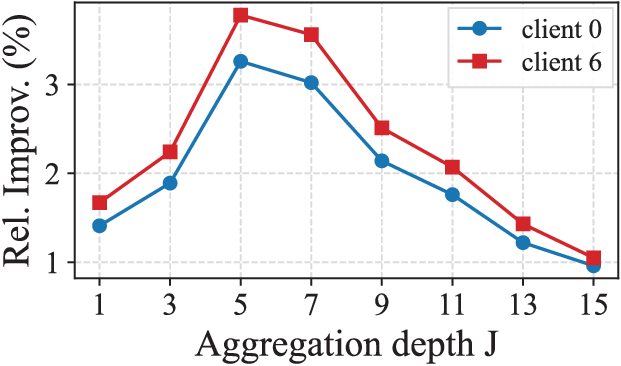}
        \caption{seg dep group} 
        \label{fig:rel-imp-seg-dep}
    \end{subfigure}
    \hfill
    \begin{subfigure}[t]{0.48\columnwidth}
        \centering
        \includegraphics[width=\linewidth]{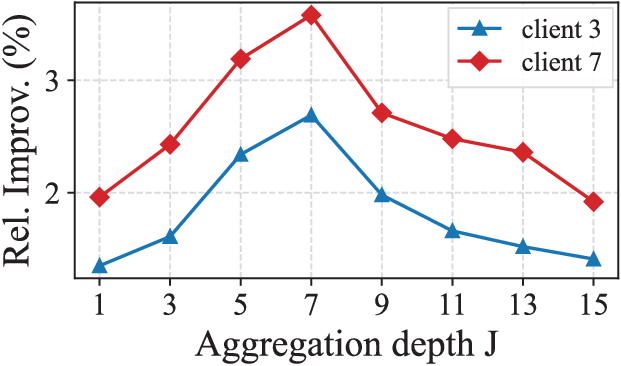}
        \caption{seg sn dep group}
        \label{fig:rel-imp-seg-sn-dep}
    \end{subfigure}

    \vspace{0.6em}
    \begin{subfigure}[t]{0.48\columnwidth}
        \centering
        \includegraphics[width=\linewidth]{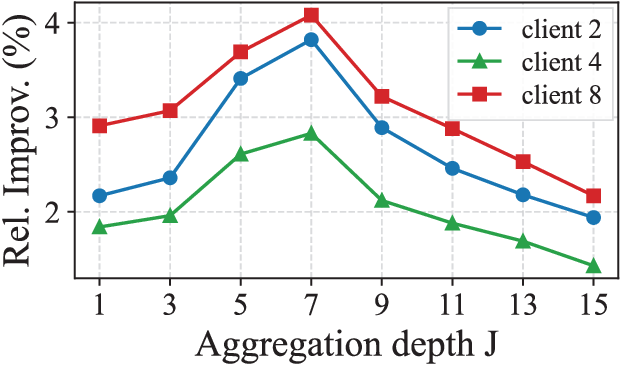}
        \caption{seg sn group}
        \label{fig:rel-imp-seg-sn}
    \end{subfigure}
    \hfill
    \begin{subfigure}[t]{0.48\columnwidth}
        \centering
        \includegraphics[width=\linewidth]{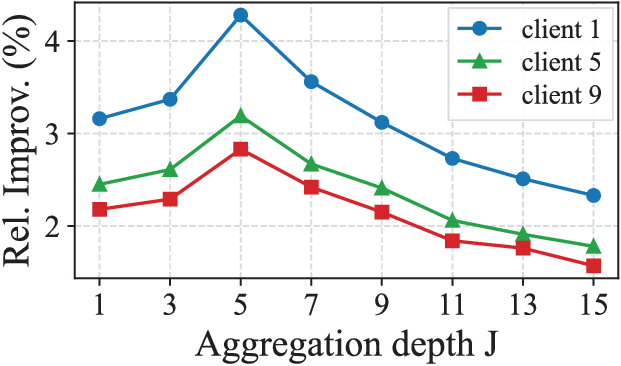}
        \caption{sn dep group}
        \label{fig:rel-imp-sn-dep}
    \end{subfigure}

    \caption{Relative improvement versus aggregation depth $J$ for different task groups.}
    \label{fig:rel-imp-all}
\end{figure}

\begin{figure}[t]
    \centering
\includegraphics[width=\linewidth]{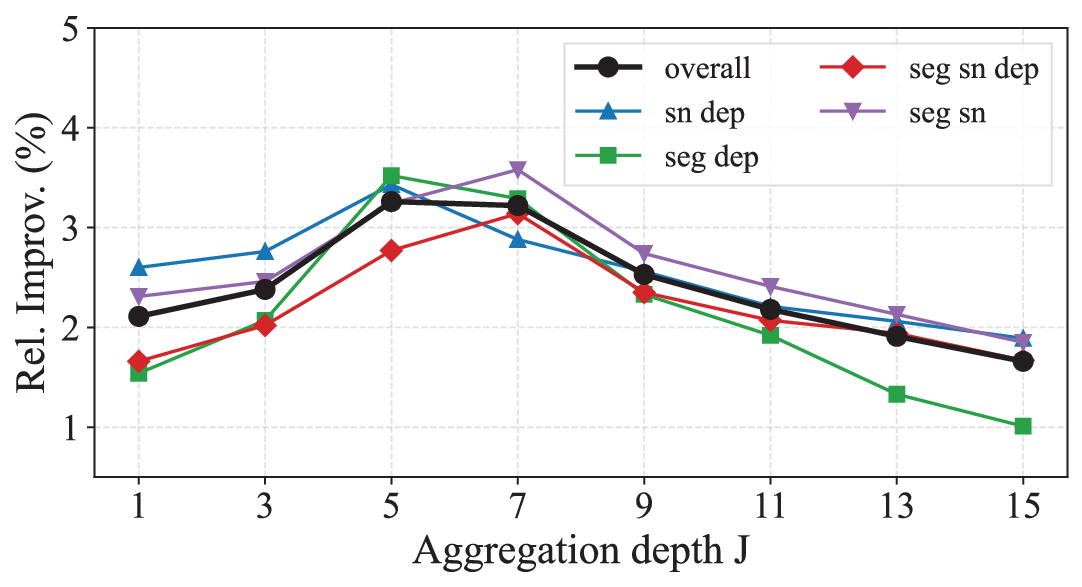}
    \caption{Overall relative improvement versus aggregation depth $J$.}
    \label{fig:overall_relative_improvement}
\end{figure}

\textbf{Empirical Validation of the $J$-dependent Trade-off.}
Theorem~\ref{Theorem: Jstar} predicts a U-shaped relationship between performance and aggregation depth. To examine this analytical prediction, Fig.~\ref{fig:rel-imp-all} plots the relative improvement versus $J$ across four different task portfolios.

A clear non-monotonic trend is observed across all four subplots. The relative improvement initially increases with $J$, reaches an empirical optimum, and then declines. At small $J$, communication is limited, and the network operates in a variance-dominated regime in which local updates cannot absorb sufficient external knowledge. Increasing the propagation depth helps reduce this variance by enabling broader neighborhood information exchange. However, overly deep aggregation pushes the system into an OCB-dominated regime, where accumulated heterogeneous parameters begin to act as structural interference. This empirical transition closely aligns with the behavior predicted by Eq.~\eqref{eq:BoundJ_abstract}.

The location of the optimal depth $J^*$ is highly sensitive to the underlying task composition. This observation is consistent with the structural heterogeneity term $\Omega_*^2$ in the derived upper bound. For the high-affinity ``sn dep'' portfolio (Fig.~\ref{fig:rel-imp-sn-dep}), collaborative gains appear immediately at $J=1$, and the post-peak degradation remains relatively mild, indicating a stronger tolerance to deeper parameter mixing. By contrast, introducing tasks with lower affinity or higher dimensionality, such as ``seg sn dep'' or ``seg sn'', substantially changes this pattern. The optimal peak shifts leftward, and the subsequent performance drop becomes much sharper. The task composition therefore limits the effective communication depth by determining the point at which heterogeneity-induced interference outweighs the marginal benefit of additional neighborhood averaging.

Client-level results further show that local dataset size affects the maximum achievable collaborative gain. Within the ``sn dep'' group, the client with the richest local data consistently attains the largest peak improvement. A larger local dataset helps anchor the optimization trajectory and makes the model more robust to heterogeneous external updates. In contrast, data-scarce clients obtain smaller peak gains and exhibit faster performance degradation as $J$ increases, indicating greater vulnerability to externally induced OCB.

\textbf{Impact of Network Size.}
We next use the Taskonomy setting to examine whether the average learning gain changes as the decentralized network becomes larger. As shown in Fig.~\ref{fig:relative_improvement_N}, the proposed method achieves the highest $\Delta_{\mathrm{all}}$ under all tested network sizes. When $N$ increases from 10 to 20, the curves of all methods remain relatively flat. Throughout this range, the changes in $\Delta_{\mathrm{all}}$ are $0.09\%$, $0.10\%$, $0.13\%$, and $0.14\%$ for FedAvg, \textit{max}, FedAMP, and the proposed similarity-aware aggregation, respectively. This indicates that, within the examined range, enlarging the decentralized network does not substantially change the average learning gain. In this case, increasing $N$ expands the global network scale, while the effective aggregation behavior is still governed by local neighborhood composition and the way neighbor updates are weighted. This variation is partly attributable to the proposed policy-driven multi-path encoder, which preserves task-dependent execution paths while allowing shared encoder blocks to learn reusable representations across tasks. Under this architecture, exchanging encoder parameters can bring collaborative gains to all aggregation rules, while the proposed similarity-aware aggregation benefits more by assigning larger weights to neighbors with more compatible routing-policy patterns.

These results suggest that the proposed method is not sensitive to network-size variation. Since the average local communication density is fixed, the main performance difference comes from whether the aggregation rule can select and weight compatible neighbor updates, rather than from the number of clients. This observation is consistent with the NYU-v2 results, where explicitly modeling task compatibility leads to stronger overall performance under heterogeneous task portfolios.

\begin{figure}[t]
\centering
\includegraphics[width=\linewidth]{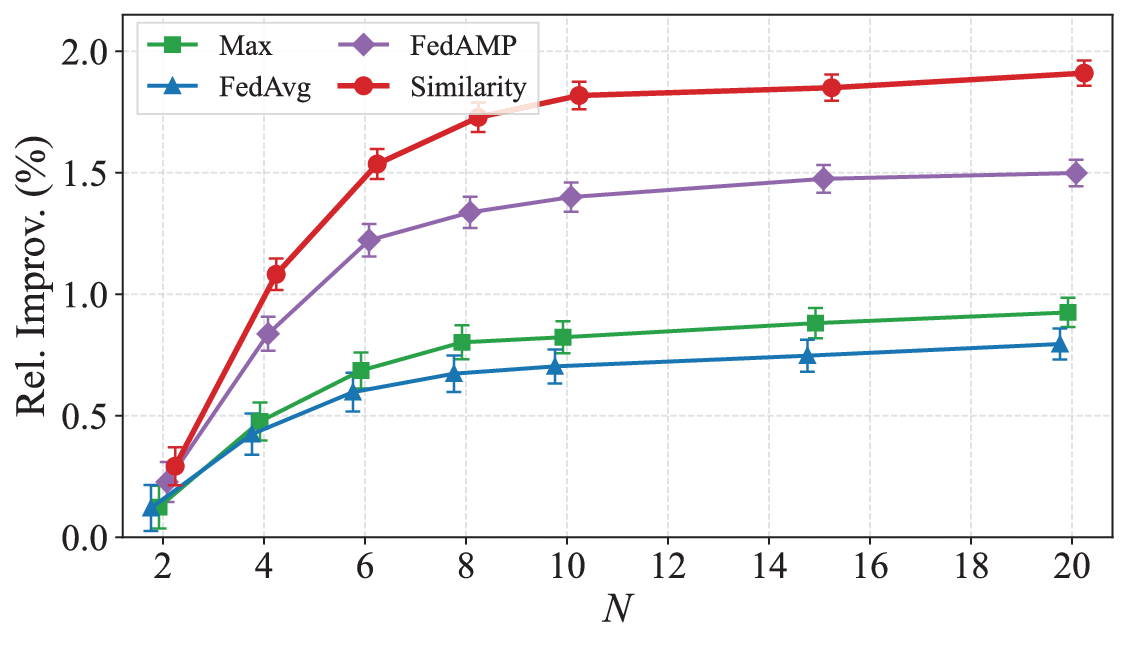}
\caption{ Overall relative improvement with varying network size. The curves show the mean over task-assignment realizations, and the error bars indicate the standard deviation.} 
\label{fig:relative_improvement_N}
\end{figure}

\textbf{Communication Cost and Impact of Regional Communication Quality.}
Using ResNet-18 as the shared backbone, FedAvg and FedAMP transmit approximately $44.71$ MB of model parameters in each client-to-neighbor exchange. The heuristic \textit{max} and the proposed similarity-aware method additionally transmit the routing-policy vectors, while their total communication volume remains approximately $44.71$ MB. For a client assigned with three tasks, these policy vectors introduce only $96$ bytes of additional payload, resulting in a marginal communication-latency overhead. Thus, the proposed method achieves improved task performance under essentially the same communication budget, with the gain primarily arising from the designed aggregation mechanism rather than a larger transmitted model.

We next examine the effect of regional channel quality on packet loss and communication latency. The nominal regional SNR is varied to control the area-level channel condition, while different network sizes lead to different client densities and one-hop link distances. We set $R=30$ and $J=1$, and count one packet transmission as one packet-time unit.

As shown in Fig.~\ref{fig:regional_snr_latency}(a), the average packet loss probability decreases as the regional SNR increases. The difference among network sizes mainly comes from the generated spatial topology. In a denser network, many one-hop links become shorter, so the packet loss probability can drop more quickly when the channel condition improves.

Fig.~\ref{fig:regional_snr_latency}(b) shows how this link reliability affects communication latency. When the regional SNR is low, retransmission dominates the latency, and larger backbones suffer more because each model exchange contains more packets. As the SNR reaches a moderate range, denser networks may benefit from shorter and more reliable links, which can offset part of the additional scheduling cost. In the high-SNR region, packet loss is already low for all network sizes, and the remaining latency difference is mainly determined by the number of scheduled links and the model size.

\begin{figure}[t]
    \centering
    \subfloat[Average packet loss ratio.\label{fig:packet_loss}]{
        \includegraphics[width=0.48\linewidth]{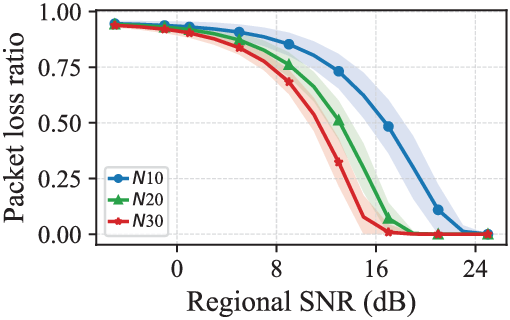}
    }
    \subfloat[Total communication latency.\label{fig:communication_latency}]{
        \includegraphics[width=0.48\linewidth]{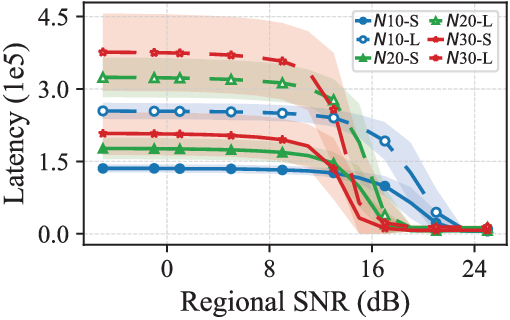}
    }
    \caption{
Effect of regional communication quality on packet loss probability and communication latency. Here, $N10$, $ N20$, and $ N30$ denote networks with $10$, $20$, and $30$ clients, respectively. $\mathrm{S}$ and $\mathrm{L}$ denote the small and large backbone settings, corresponding to ResNet-18 and ResNet-34.
The curves show the mean over random topology realizations, and the shaded regions indicate the 10th--90th percentile range.}

    \label{fig:regional_snr_latency}
\end{figure}

\section{Conclusion}
This paper investigated how structural heterogeneity affects decentralized multi-task optimization in DFL systems. Our analysis revealed a U-shaped trade-off in the steady-state error: increasing the aggregation depth reduces topology-dependent variance, but also amplifies OCB. To address this issue, we proposed a similarity-aware routing strategy that selectively filters structurally mismatched updates, enabling more reliable and effective decentralized collaboration.

Experiments on the NYU-v2 dataset support the theoretical analysis. When operated at the analytically derived optimal depth, the proposed framework mitigates negative transfer and consistently outperforms decentralized FedAvg, FedAMP and heuristic baselines, achieving a $4.77\%$ global relative improvement. Results on Taskonomy further show that the proposed aggregation strategy is only mildly affected by moderate changes in network size.

Future work will extend the framework to more realistic edge settings, including time-varying wireless channels, dynamic mesh topologies, and adaptive online schemes for aggregation depth control.

\appendix
\subsection{Proof of Lemma~\ref{lemma: grad_and_disagree_bound}}\label{Proof: grad_and_disagree_bound}
In this proof, we derive an explicit cross-round recursion for the consensus error $E_{\mathrm{cons},r+1,J}$ in terms of $E_{\mathrm{cons},r,J}$ by combining a local-training disagreement expansion bound and a local-aggregation disagreement bound.

\paragraph{Local Training}
Let $\theta_{i,r,e}$ be the local model of client $i$ at step $e$ of round $r$. We define the \emph{pre-aggregation} consensus error at the end of local training as
\begin{equation}\label{eq:app_Econs_r0_def_recall}
    E_{\mathrm{cons},r,0} \triangleq \frac{1}{N}\sum_{i\in\mathcal N}\big\|\theta_{i,r,E}-\bar{\theta}_{r,E}\big\|^{2},
\end{equation}
where $\bar{\theta}_{r,E} \triangleq \frac{1}{N}\sum_{i\in\mathcal{N}} \theta_{i,r,E}$. Note that at the start of the round, $e=0$, the error is $E_{\mathrm{cons},r-1,J}$.

We first bound the expected drift of the local updates. The update rule is $\theta_{i,r,e+1} = \theta_{i,r,e} - \eta g_{i,r,e}$. The mean model evolves as $\bar{\theta}_{r,e+1} = \bar{\theta}_{r,e} - \eta \bar{g}_{r,e}$, where $\bar{g}_{r,e} \triangleq \frac{1}{N}\sum_i g_{i,r,e}$.
The deviation from the mean evolves as
\begin{equation}\label{eq:dev_recursion}
    \theta_{i,r,e+1} - \bar{\theta}_{r,e+1} = (\theta_{i,r,e} - \bar{\theta}_{r,e}) - \eta (g_{i,r,e} - \bar{g}_{r,e}).
\end{equation}
Squaring the norm and taking expectations yields the expansion given by 
\begin{align}\label{eq:dev_sq_expanded}
    \mathbb{E}\|\theta_{i,r,e+1} \!-\! \bar{\theta}_{r,e+1}\|^2 &= \mathbb{E}\|\theta_{i,r,e} \!\!- \!\!\bar{\theta}_{r,e}\|^2 + \eta^2 \mathbb{E}\|g_{i,r,e} \!\!-\!\! \bar{g}_{r,e}\|^2 \nonumber \\
    &\hspace{-1em} - 2\eta \mathbb{E}\langle \theta_{i,r,e} - \bar{\theta}_{r,e}, g_{i,r,e} - \bar{g}_{r,e} \rangle.
\end{align}
Using Young's inequality $\|a-b\|^2 \le (1+\beta)\|a\|^2 + (1+\frac{1}{\beta})\|b\|^2$, we simplify the analysis by bounding the expected deviation as
\begin{align}\label{eq:dev_youngs}
    \mathbb{E}\|\theta_{i,r,e+1} - \bar{\theta}_{r,e+1}\|^2 &\le (1+\eta^2 E \mathrm{L}_{\mathrm{eff}}^2)\mathbb{E}\|\theta_{i,r,e} - \bar{\theta}_{r,e}\|^2 \nonumber \\
    &\hspace{-5em} + \Big(1+\frac{1}{\eta^2 E \mathrm{L}_{\mathrm{eff}}^2}\Big)\eta^2 \mathbb{E}\|g_{i,r,e} - \bar{g}_{r,e}\|^2.
\end{align}
To handle the gradient variance term, we decompose it using the exact mean gradient $\nabla \bar{\mathcal{L}}(\theta_{r,e}) \triangleq \frac{1}{N}\sum_{j\in \mathcal{N}} \nabla \mathcal{L}_j(\theta_{j,r,e})$ expressed as
\begin{align}\label{eq:grad_decomp}
    g_{i,r,e} - \bar{g}_{r,e} &= (\nabla \mathcal{L}_i(\theta_{i,r,e}) - \nabla \bar{\mathcal{L}}(\theta_{r,e})) \nonumber \\
    & \hspace{-5em}+g_{i,r,e} - \nabla \mathcal{L}_i(\theta_{i,r,e}) - \frac{1}{N}\sum_{j\in \mathcal{N} }(g_{j,r,e} - \nabla \mathcal{L}_{j}(\theta_{j,r,e})) .
\end{align}
Applying Jensen's inequality and Assumption~\ref{assumption: unbiased gradient}, we bound the expected squared norm of the stochastic noise by $4(\sigma_i^2 + \sigma^2)$. Similarly, the heterogeneity drift is bounded by averaging over all clients $i \in \mathcal{N}$ via the same inequality,
\begin{align}
    &\frac{1}{N}\!\!\sum_{i\in\mathcal{N}}\! \|\nabla\! \mathcal{L}_i(\theta_{i,r,e}\!)\!\!-\!\! \nabla\! \bar{\mathcal{L}}(\theta_{r,e\!})\|^2 \!\!\le\!\! \frac{3}{N}\!\!\sum_{i\in\mathcal{N}}\!\!\|\nabla\! \mathcal{L}_i(\theta_{i,r,e}\!)\!\!-\!\!\! \nabla \!\mathcal{L}_i(\bar{\theta}_{r,e})\|^2 \nonumber \\&+\!\! 3\|\nabla \bar{\mathcal{L}}(\bar{\theta}_{r,e})\!-\!\! \nabla \bar{\mathcal{L}}(\theta_{r,e})\|^2\!\!  + \!\!\frac{3}{N}\sum_{i\in\mathcal{N}} \|\nabla \mathcal{L}_i(\bar{\theta}_{r,e})\!-\!\! \nabla \bar{\mathcal{L}}(\bar{\theta}_{r,e})\|^2 \nonumber \\
    &\le 6\mathrm{L}_{\mathrm{eff}}^2 E_{\mathrm{cons},r,e} + 3\widehat{\Gamma}.
\end{align}
Substituting these bounds back into \eqref{eq:dev_youngs}, recursively applying it over $e=0,\dots,E-1$, and defining appropriate bounding constants $a_1$ and $a_2$ leads to
\begin{align}\label{eq:local_expansion_final}
    E_{\mathrm{cons},r,0}\!\!\le\!\! a_2 \eta^2\! E^2\! (\widehat{\Gamma}\!\! +\!\! \sigma^2)\!\! +\!\! (1\!\!+\!\!a_1 \!\eta^2\! E^2 \mathrm{L}_{\mathrm{eff}}^2\!)\! E_{\mathrm{cons},r-1,J}.
\end{align}
This bound quantifies how statistical heterogeneity and gradient noise introduce disagreement into the system prior to the aggregation step. 

\paragraph{Local Aggregation}
Next, we determine how the disagreement contracts after $J$ aggregation steps. Let $\boldsymbol{\Psi}_{r,0} \in \mathbb{R}^{N \times M}$ be the matrix of local models before aggregation. The aggregation proceeds as $\boldsymbol{\Psi}_{r,J} = \boldsymbol{\mathrm{W}}^J \boldsymbol{\Psi}_{r,0}$. Since $\boldsymbol{\mathrm{W}}$ is strictly column-stochastic, it perfectly preserves the arithmetic mean, yielding
\begin{equation}
    \bar{\psi}_{r,J} = \frac{1}{N}\mathbf{1}^\top \boldsymbol{\Psi}_{r,J} = \frac{1}{N}\mathbf{1}^\top \boldsymbol{\mathrm{W}}^J \boldsymbol{\Psi}_{r,0} = \bar{\psi}_{r,0}.
\end{equation}
The consensus error is formulated as $E_{\mathrm{cons},r,J} = \frac{1}{N}\|(\mathbf{I} - \mathbf{M})\boldsymbol{\Psi}_{r,J}\|_F^2$, where $\mathbf{M} \triangleq \frac{1}{N}\mathbf{1}\mathbf{1}^\top$ denotes the mean-centering projection matrix. To capture the impact of dispersed client optima $\Theta^* = [\theta_1^*, \dots, \theta_N^*]^\top$, we decompose the error relative to the centered optima matrix as
\begin{align}\label{eq:gossip_decomp_structural}
    (\mathbf{I} \!\!- \!\!\mathbf{M})\boldsymbol{\Psi}_{r,J} \!\!=\!\!(\mathbf{I} \!\!-\!\! \mathbf{M})\boldsymbol{\mathrm{W}}^J \Theta^* \!\!+\!\! (\mathbf{I} \!\!-\!\! \mathbf{M})\boldsymbol{\mathrm{W}}^J \!(\boldsymbol{\Psi}_{r,0}\!\! -\!\! \Theta^*\!).
\end{align}
Applying the generalized triangle inequality $\|A+B\|_F^2 \le (1+\gamma)\|A\|_F^2 + (1+\frac{1}{\gamma})\|B\|_F^2$ for any $\gamma > 0$ yields
\begin{align}
    E_{\mathrm{cons},r,J} &\le \frac{1+\gamma}{N} \|(\mathbf{I} - \mathbf{M})\boldsymbol{\mathrm{W}}^J (\boldsymbol{\Psi}_{r,0} - \Theta^*)\|_F^2 \nonumber \\ 
    &\quad + \frac{1+1/\gamma}{N} \|(\mathbf{I} - \mathbf{M})\boldsymbol{\mathrm{W}}^J \Theta^*\|_F^2.
\end{align}
The stochastic matrix contracts the disagreement component orthogonal to the average-consensus subspace at a geometric rate $\rho$, i.e.,
\begin{equation}
    \|(\mathbf{I} - \mathbf{M})\mathbf{W}^J(\mathbf{I}-\mathbf{M}) X\|_F \le \rho^J \|(\mathbf{I} - \mathbf{M})X\|_F.
\end{equation}
Since $\mathbf{W}$ is column-stochastic but not necessarily doubly stochastic, applying $\mathbf{W}^J$ to a matrix with a nonzero mean component may additionally induce an imbalance residual $(\mathbf{I}-\mathbf{M})\mathbf{W}^J\mathbf{M}X$. In particular, when the local optima $\Theta^*$ are dispersed across clients, $\mathbf{W}^J\Theta^*$ may exhibit a persistent structural shift. While $\mathbf{W}^J$ suppresses network disagreement~\cite{hashemi2022benefits}, it may simultaneously displace the iterates away from their respective empirical minimizers. Leveraging spectral decay, this OCB floor scales with the structural dispersion $\Omega_*^2$ as
\begin{equation}\label{eq:ocb_scaling_bound_app}
    \Delta_{\mathrm{OCB}}(J) \le c_{\mathrm{OCB}}(1-\rho^{2J})\Omega_*^2.
\end{equation}
The finalized consensus error bound after $J$ mixing steps is thus given by
\begin{equation}\label{eq:gossip_bound_final}
    E_{\mathrm{cons},r,J} \le \rho^{2J} E_{\mathrm{cons},r,0} + c_{\mathrm{OCB}}(1-\rho^{2J})\Omega_*^2.
\end{equation}
This characterization reveals a dual effect: the first term suppresses initial disagreement and gradient noise, while the second introduces an irreducible bias floor dictated by the decentralized network's structural heterogeneity.

Substituting the local expansion bound \eqref{eq:local_expansion_final} into the gossip contraction bound \eqref{eq:gossip_bound_final}, we obtain
\begin{align}
    E_{\mathrm{cons},r+1,J} &\le  \rho^{2J}(1+a_1 \eta^2 E^2 \mathrm{L}_{\mathrm{eff}}^2) E_{\mathrm{cons},r,J} \nonumber \\
    &\hspace{-4em} + \rho^{2J} a_2 \eta^2 E^2 (\widehat{\Gamma} + \sigma^2) + c_{\mathrm{OCB}}(1-\rho^{2J})\Omega_*^2.
\end{align}
This explicitly separates the expansion, contraction, and bias terms, concluding the proof.
\subsection{Proof of Lemma~\ref{lemma: decrease of potential}}\label{Proof: decrease of potential}
We analyze the per-round drift $\Delta \Phi_r \triangleq \mathbb{E}[\Phi_{r+1,J} - \Phi_{r,J}]$ of the Lyapunov function $\Phi_{r,J} = \mathcal{L}(\bar{\theta}_{r,E}) + \alpha E_{\mathrm{cons},r,J}$. Applying $\mathrm{L}_{\mathrm{eff}}$-smoothness to $\mathcal{L}$, the expected change over $E$ local steps is given by
\begin{align}\label{eq:smoothness_expansion}
    \mathbb{E}[\mathcal{L}(\bar{\theta}_{r+1,E}) - \mathcal{L}(\bar{\theta}_{r,E})] &\le   \frac{\mathrm{L}_{\mathrm{eff}}}{2} \mathbb{E}\|\bar{\theta}_{r+1,E} - \bar{\theta}_{r,E}\|^2 \nonumber \\
    &\hspace{-5em} +\mathbb{E}\langle \nabla \mathcal{L}(\bar{\theta}_{r,E}), \bar{\theta}_{r+1,E} - \bar{\theta}_{r,E} \rangle.
\end{align}
Standard bounding for decentralized settings yields the descent in terms of gradient norm, variance, and model disagreement as
\begin{align}\label{eq:loss_descent_explicit}
    \mathbb{E}[\mathcal{L}(\bar{\theta}_{r+1,E}) - \mathcal{L}(\bar{\theta}_{r,E})] &\le -c_{\Phi,1} \eta E \mathbb{E}\|\nabla \mathcal{L}(\bar{\theta}_{r,E})\|^2 \nonumber \\
    &\hspace{-10em}+\!\! \frac{C_{\mathcal{L}} \eta \mathrm{L}_{\mathrm{eff}}^2}{N} \!\!\sum_{e,i}\! \mathbb{E}\|\theta_{i,r+1,e}\!\! -\!\! \bar{\theta}_{r,E}\|^2\!\!  +\!\! C_{\mathcal{L}} \eta E (\!\sigma^2\!\! +\!\! \widehat{\Gamma}),
\end{align}
where $c_{\Phi,1}, C_{\mathcal{L}} > 0$. Decoupling the summation term via the triangle inequality gives
\begin{align}\label{eq:intra_decompose}
    \mathbb{E}\|\theta_{i,r+1,e} - \bar{\theta}_{r,E}\|^2 &\le 2\mathbb{E}\|\theta_{i,r+1,0} - \bar{\theta}_{r,E}\|^2 \nonumber \\&+ 2\mathbb{E}\|\theta_{i,r+1,e} - \theta_{i,r+1,0}\|^2.
\end{align}
The first term evaluates to $2\mathbb{E}E_{\mathrm{cons},r,J}$ since $\theta_{i,r+1,0} \equiv \psi_{i,r,J}$. For the second, unfolding local updates and applying Jensen's inequality yields
\begin{equation}\label{eq:drift_jensen_app}
    \mathbb{E}\|\theta_{i,r+1,e} - \theta_{i,r+1,0}\|^2 \le \eta^2 e \sum_{k=0}^{e-1} \mathbb{E}\|g_{i,r+1,k}\|^2.
\end{equation}
Using Lemma~\ref{lemma: grad_and_disagree_bound}, the average local drift is bounded as
\begin{equation}\label{eq:drift_bound_avg}
    \frac{2}{EN} \sum_{e,i} \mathbb{E}\|\theta_{i,r+1,e} - \theta_{i,r+1,0}\|^2 \le 2 a_2 \eta^2 E^2 (\widehat{\Gamma} + \sigma^2).
\end{equation}
Combining these components strictly bounds the expected disagreement in \eqref{eq:loss_descent_explicit} by $2 \mathbb{E} E_{\mathrm{cons},r,J} + 2 a_2 \eta^2 E^2 (\widehat{\Gamma} + \sigma^2)$. Substituting this bound and the consensus recursion from Lemma~\ref{lemma: grad_and_disagree_bound} into the total drift definition yields
\begin{align}
    \mathbb{E}[\Delta \Phi_r] &= \!\mathbb{E}[\mathcal{L}(\bar{\theta}_{r+1,E})\!\! -\!\! \mathcal{L}(\bar{\theta}_{r,E})]\!\! +\!\! \alpha \mathbb{E}[E_{\mathrm{cons},r+1,J} \!\!-\!\! E_{\mathrm{cons},r,J}] \nonumber \\
    &\le -c_{\Phi,1} \eta E \mathbb{E}\|\nabla \mathcal{L}(\bar{\theta}_{r,E})\|^2 + 2 C_{\mathcal{L}} \eta E \mathrm{L}_{\mathrm{eff}}^2 \mathbb{E} E_{\mathrm{cons},r,J} \nonumber \\
    & + C_{\mathcal{L}} \eta E (\sigma^2 + \widehat{\Gamma}) + 2 C_{\mathcal{L}} a_2 \eta^3 E^3 \mathrm{L}_{\mathrm{eff}}^2 (\widehat{\Gamma} + \sigma^2) \nonumber \\
    & + \alpha \Big( \rho^{2J}(1+a_1 \eta^{2} E^{2} \mathrm{L}_{\mathrm{eff}}^{2}) - 1 \Big) \mathbb{E} E_{\mathrm{cons},r,J} \nonumber \\
    & +\! \alpha \rho^{2J} a_2 \eta^2 E^2 \!(\widehat{\Gamma}\!\! +\! \sigma^2)\! \!+\! \alpha c_{\mathrm{OCB}}(1\!-\!\rho^{2J})\Omega_{*}^{2}.
\end{align}
Defining $c_{\Phi,2} \triangleq C_{\mathcal{L}} + 2 C_{\mathcal{L}} a_2 \eta^2 E^2 \mathrm{L}_{\mathrm{eff}}^2 + \alpha \rho^{2J} a_2 \eta E$ consolidates noise and heterogeneity. To ensure monotonicity, the coefficient of $\mathbb{E} E_{\mathrm{cons},r,J}$ must be strictly negative, which requires
\begin{equation}\label{eq:alpha_range_proof_final}
    \alpha \!\Big(\!\! (1\!\!-\!\!\rho^{2J})\!\! -\!\! \rho^{2J} \!a_1 \eta^2\! E^2 \mathrm{L}_{\mathrm{eff}}^2\! \!\Big)\!\!\! -\!\! 2 C_{\mathcal{L}} \eta E \mathrm{L}_{\mathrm{eff}}^2\!\ge\! c_{\Phi,3}(1\!\!-\!\!\rho^{2J}).
\end{equation}
Defining $c_{\Phi,3} \triangleq \alpha ( 1 - \frac{\rho^{2J} a_1 \eta^2 E^2 \mathrm{L}_{\mathrm{eff}}^2}{1-\rho^{2J}} ) - \frac{2 C_{\mathcal{L}} \eta E \mathrm{L}_{\mathrm{eff}}^2}{1-\rho^{2J}}$ and choosing a sufficiently large $\alpha$ to ensure $c_{\Phi,3} \in (0,1)$, the per-round drift is effectively bounded as
\begin{align}
    \mathbb{E}[\Delta \Phi_r] &\le \!\!-c_{\Phi,1}\eta E\,\mathbb{E}\big\|\nabla\mathcal{L}(\bar{\theta}_{r,E})\big\|^{2} \!\!-\!c_{\Phi,3}(1\!-\!\rho^{2J})\,\mathbb{E}E_{\mathrm{cons},r,J} \nonumber \\
    &\quad + c_{\Phi,2}\eta E(\widehat{\Gamma} +\sigma^{2}) +\alpha c_{\mathrm{OCB}}(1-\rho^{2J})\,\Omega_{*}^{2}.
\end{align}
This completes the proof.
\subsection{Proof of Theorem~\ref{Theorem: linear}}\label{Proof: Theorem linear}
Building on Lemma~\ref{lemma: decrease of potential}, we establish linear convergence to a steady-state error floor. To explicitly characterize the trade-off, we decouple the fundamental gradient noise from the consensus-induced noise, yielding the refined expected per-round drift formulated as
\begin{align}\label{eq:app_one_step_drift_linear_new}
    \Delta \Phi_r\!&\le\!\!-c_{\Phi,1}\eta E\mathbb{E}\big\|\nabla\!\mathcal{L}\!(\bar{\theta}_{r,E})\!\big\|^{2} \!\!-\!c_{\Phi,3}(1\!-\!\rho^{2J})\mathbb{E}E_{\mathrm{cons},r,J} \nonumber\\
    &\hspace{-2em} +\! \eta^2 E (c_{\text{base}}\! +\! \alpha \rho^{2J} c_{\text{cons}}) \!(\widehat{\Gamma}\!\!+\!\sigma^{2})\! +\! \alpha c_{\mathrm{OCB}}(1\!-\!\rho^{2J})\Omega_{*}^{2},
\end{align}
where $c_{\text{base}}$ characterizes the intrinsic SGD noise, and $c_{\text{cons}}$ captures the residual disagreement noise proportional to the spectral contraction $\rho^{2J}$.

Applying the PL condition from Assumption~\ref{assumption: PL} to the mean parameter $\bar{\theta}_{r,E}$ yields
\begin{align}\label{eq:app_one_step_drift_PL_new}
    \Delta \Phi_r &\!\le\!\! -2c_{\Phi,1}\mu_{\mathrm{eff}}\eta E\mathbb{E}[\mathcal{L}(\bar{\theta}_{r,E})\!\!-\!\!\mathcal{L}^{*}] \!\!-\!\!c_{\Phi,3}(1\!\!-\!\!\rho^{2J})\,\mathbb{E}E_{\mathrm{cons},r,J} \nonumber \\
    &\hspace{-2em} + \eta^2 E (c_{\text{base}}\! +\! \alpha \rho^{2J} c_{\text{cons}}) (\widehat{\Gamma}\!\!+\!\sigma^{2})\! \! +\!\alpha c_{\mathrm{OCB}}(1\!-\!\rho^{2J})\Omega_{*}^{2}.
\end{align}

We define the linear contraction factor for the composite Lyapunov function as
\begin{equation}\label{eq:app_kappa_lin_def_new}
    \kappa_{\mathrm{lin}} \triangleq c_{\Phi,1}\mu_{\mathrm{eff}}\eta E.
\end{equation}
By enforcing the admissible condition in \eqref{eq:alpha range}, we ensure that the consensus error contracts at least as fast as the global objective, satisfying $\kappa_{\mathrm{lin}}\alpha \le c_{\Phi,3}(1-\rho^{2J})$. Furthermore, since the expected loss gap is strictly non-negative, we can safely bound $-2\kappa_{\mathrm{lin}} \le -\kappa_{\mathrm{lin}}$. Consequently, the negative descent terms in \eqref{eq:app_one_step_drift_PL_new} can be merged and bounded by the composite function as
\begin{align}\label{eq:app_absorb_to_composite_new} 
    \!\!\!-\!2\kappa_{\mathrm{lin}}\mathbb{E}[\mathcal{L}\!(\bar{\theta}_{r,E}\!)\!\!-\!\!\mathcal{L}^{*}\!]\!\! -\!\! \kappa_{\mathrm{lin}}\alpha \mathbb{E}\!E_{\mathrm{cons},r,J}\!\! \le\!\!  -\!\!\kappa_{\mathrm{lin}}\mathbb{E}[\Phi_{r,J}\!\!-\!\!\mathcal{L}^{*}\!].
\end{align}

Substituting \eqref{eq:app_absorb_to_composite_new} back into \eqref{eq:app_one_step_drift_PL_new} and rearranging the terms using $\Delta \Phi_r = \mathbb{E}[\Phi_{r+1,J}] - \mathbb{E}[\Phi_{r,J}]$, we establish the linear recurrence relation formulated as
\begin{align}\label{eq:app_linear_recursion_new}
    \mathbb{E}\Big[\Phi_{r+1,J}&-\mathcal{L}^{*}\Big] \le (1-\kappa_{\mathrm{lin}})\, \mathbb{E}\Big[\Phi_{r,J}-\mathcal{L}^{*}\Big] \nonumber\\
    &\hspace{-4em} +\!\!\eta^2 E (c_{\text{base}} \!\!+\! \alpha \rho^{2J} c_{\text{cons}}) (\widehat{\Gamma}\!\!+\!\sigma^{2})\!\!  +\!\! \alpha c_{\mathrm{OCB}}(1\!\!-\!\!\rho^{2J})\Omega_{*}^{2}.
\end{align}
Let the sum of the last two additive terms be denoted as $\text{Total Noise}$. Iterating this recursion from round $0$ to $r$ yields the standard linear convergence bound given by
\begin{equation}\label{eq:app_linear_iterated_new}
    \mathbb{E}\Big[\Phi_{r,J}-\mathcal{L}^{*}\Big] \le (1-\kappa_{\mathrm{lin}})^{r}\Big(\Phi_{0,J}-\mathcal{L}^{*}\Big) + \frac{\text{Total Noise}}{\kappa_{\mathrm{lin}}}.
\end{equation}

To explicitly characterize the trade-off induced by the gossip depth $J$, we decouple the steady-state error floor (i.e., the second term in \eqref{eq:app_linear_iterated_new}) into its fundamental statistical components. By expanding $\kappa_{\mathrm{lin}}$ and algebraically factoring out $(1-\rho^{2J})$ in the denominator of the consensus variance term to highlight the spectral gap dependence, the total steady-state error is strictly bounded by three distinct components expressed as
\begin{align}\label{eq:app_steady_state_final}
    \text{Error Floor} &=\!\! \frac{\eta c_{\text{base}}}{c_{\Phi,1}\mu_{\mathrm{eff}}}(\widehat{\Gamma}\!+\!\sigma^{2}) \!+ \!\frac{\eta^2 c_{\text{var}}}{c_{\Phi,1}\mu_{\mathrm{eff}}(1\!-\!\rho^{2J})}(\widehat{\Gamma}\!\!+\!\sigma^{2}) \nonumber \\
    &\quad + \frac{\alpha c_{\mathrm{OCB}}(1-\rho^{2J})}{c_{\Phi,1}\mu_{\mathrm{eff}}}\Omega_{*}^{2}.
\end{align}
By defining the bounding constants as $C_{\text{base}} = \frac{\eta c_{\text{base}}}{c_{\Phi,1}\mu_{\mathrm{eff}}}(\widehat{\Gamma}+\sigma^{2})$, $C_{\Gamma} = \frac{\eta^2 c_{\text{var}}}{c_{\Phi,1}\mu_{\mathrm{eff}}}(\widehat{\Gamma}+\sigma^{2})$, and $C_{\Omega} = \frac{c_{\mathrm{OCB}}}{c_{\Phi,1}\mu_{\mathrm{eff}}}$, we recover the exact theoretical upper bound presented in Theorem~\ref{Theorem: linear}. This structural form mathematically isolates the dual impact of $J$, establishing the foundation for optimizing the communication depth. This completes the proof.
\subsection{Proof of Theorem~\ref{Theorem: Jstar}}\label{Proof: Theorem Jstar}
From Theorem~\ref{Theorem: linear}, the transient term vanishes as $r\to\infty$. Ignoring the $J$-independent noise $C_{\text{base}}$, we minimize the steady-state objective formulated as
\begin{equation}\label{eq:app_boundJ_compact}
    \mathrm{Bound}(J) \triangleq \frac{A}{1-\rho^{2J}} + \alpha B (1-\rho^{2J}),
\end{equation}
where $A \triangleq C_{\Gamma}$ and $B \triangleq C_{\Omega}\Omega_{*}^{2}$. To find the stationary point, we relax $J$ and substitute $y \triangleq 1-\rho^{2J} \in (0,1)$, yielding $f(y) \triangleq \frac{A}{y} + \alpha B y$. Differentiating $f(y)$ yields
\begin{equation}\label{eq:app_fy_derivs}
    f'(y) = \alpha B - \frac{A}{y^2}, \qquad f''(y) = \frac{2A}{y^3}.
\end{equation}
Because $A>0$ and $y>0$, the second derivative $f''(y) > 0$ confirms strict convexity. Setting $f'(y)=0$ identifies the unique stationary point $y^* = \sqrt{A/(\alpha B)}$, which represents a feasible interior solution ($y^* < 1$) if $\alpha B > A$. Mapping $y^*$ back to $J^*$ using $1-\rho^{2J^*} = y^*$ results in $\rho^{2J^*} = 1 - \sqrt{A/(\alpha B)}$. Solving for $J^*$ yields the closed-form expression given by
\begin{equation}\label{eq:app_J_star_final}
    J^* = \frac{1}{-2\ln\rho} \ln\left(\frac{1}{1 - \sqrt{\frac{A}{\alpha B}}}\right).
\end{equation}

Strict convexity and the monotonic relationship between $y$ and $J$ ensure that $J^*$ is the unique global minimizer, mathematically confirming the U-shaped performance trend. This completes the proof.

\bibliographystyle{IEEEtran}
\bibliography{citations}

\begin{IEEEbiography}[{\includegraphics[width=1in,height=1.25in,clip,keepaspectratio]{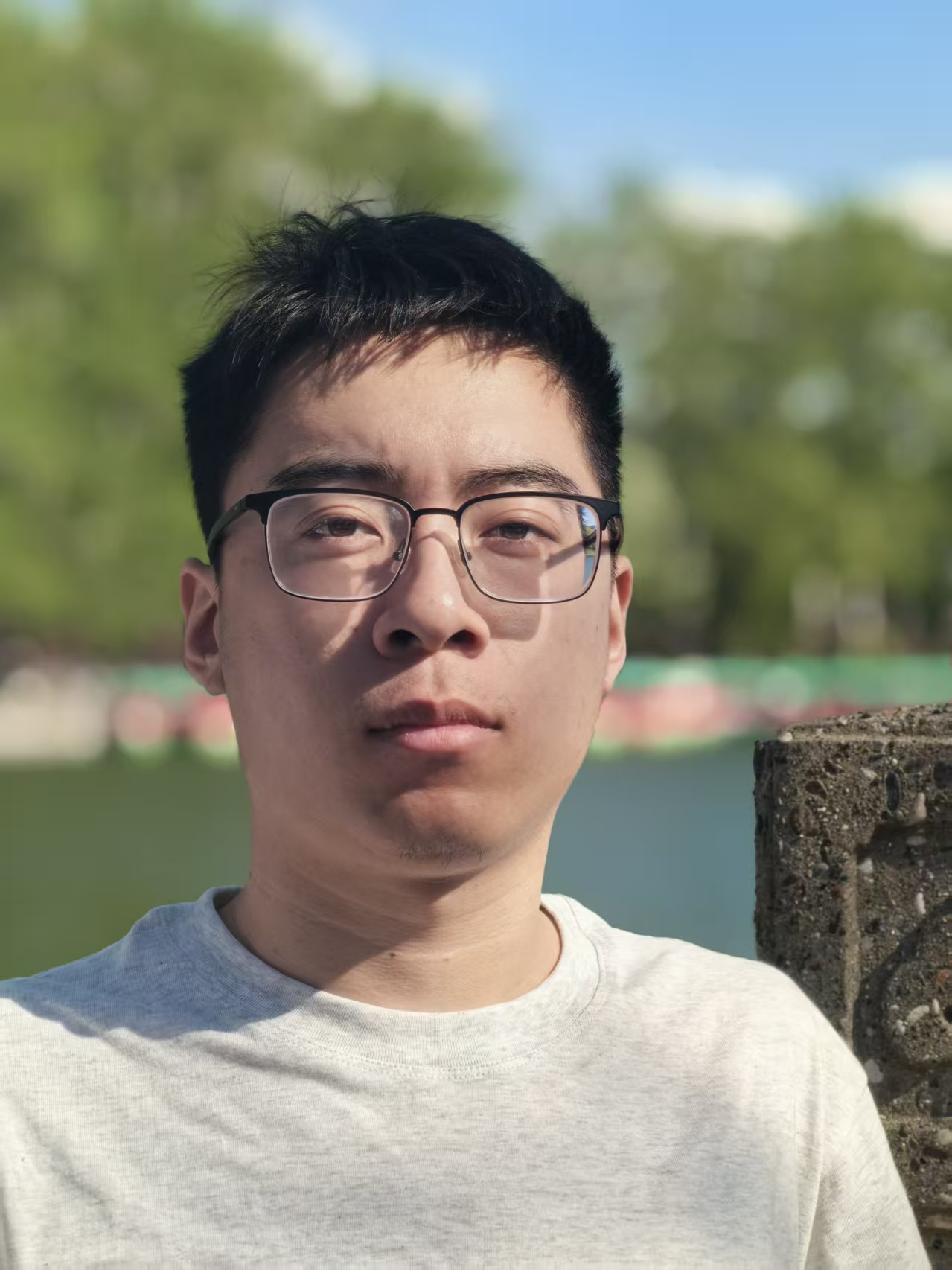}}]{Lin Yin} 
received the B.E. degree in applied physics from Beijing University of Posts and Telecommunications (BUPT), China, in 2023. He is pursuing his Ph.D. with the School of Information and Communication Engineering at BUPT. His research interests include personalized federated learning and semantic communication.
\end{IEEEbiography}\vspace{-20 mm}

\begin{IEEEbiography}[{\includegraphics[width=1in,height=1.25in,clip,keepaspectratio]{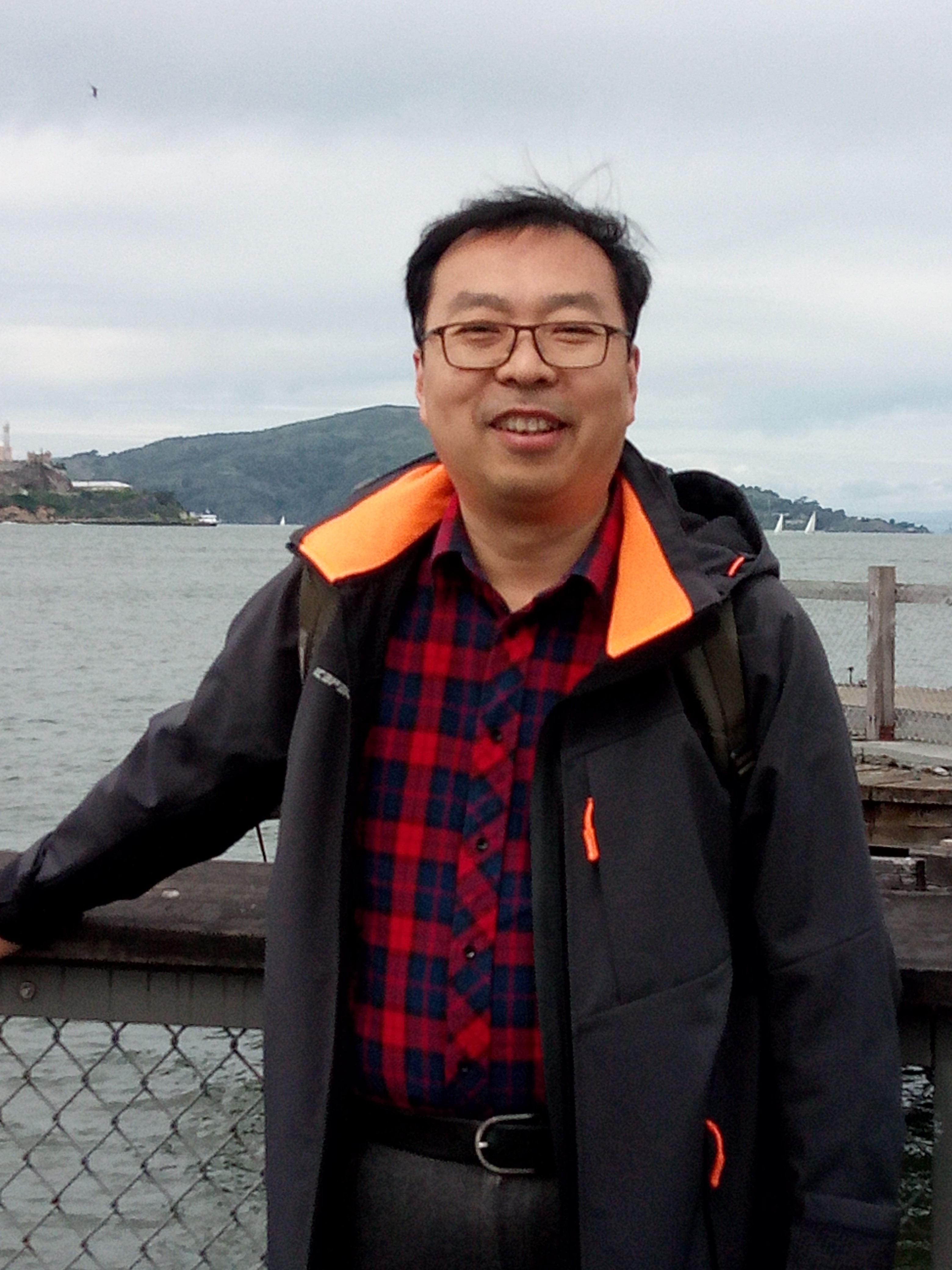}}]{Tiejun Lv} received the M.S. and Ph.D. degrees in electronic engineering from the University of Electronic Science and Technology of China (UESTC), Chengdu, China, in 1997 and 2000, respectively. From January 2001 to January 2003, he was a Post-Doctoral Fellow at Tsinghua University, Beijing, China. In 2005, he was promoted to a Full Professor at the School of Information and Communication Engineering, Beijing University of Posts and Telecommunications (BUPT). From September 2008 to March 2009, he was a Visiting Professor with the Department of Electrical Engineering, Stanford University, Stanford, CA, USA. He is currently the author of four books, one book chapter, and more than 170 published journal articles and 230 conference papers on the physical layer of wireless mobile communications. His current research interests include signal processing, communications theory, and networking. He was a recipient of the Program for New Century Excellent Talents in University Award from the Ministry of Education, China, in 2006. He received the Nature Science Award from the Ministry of Education of China for the hierarchical cooperative communication theory and technologies in 2015 and Shaanxi Higher Education Institutions Outstanding Scientific Research Achievement Award in 2025.
\end{IEEEbiography}\vspace{-10 mm}

\begin{IEEEbiography}[{\includegraphics[width=1in,height=1.25in,clip,keepaspectratio]{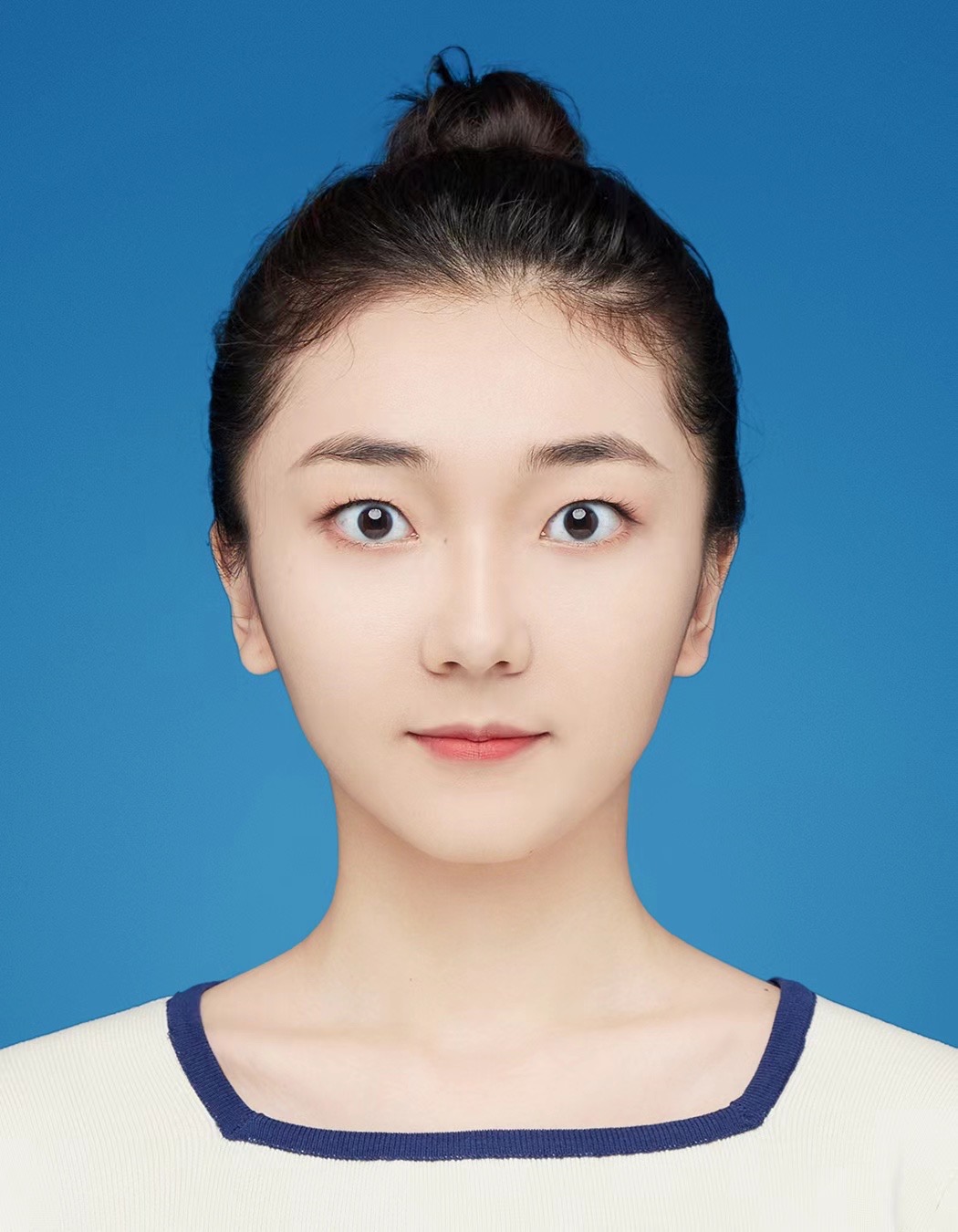}}]{Weicai Li} received the B.E. and Ph.D. degrees from the School of Information and Communication Engineering, Beijing University of Posts and Telecommunications, China, in 2020 and 2025, respectively. From December 2022 to December 2023, she was a Visiting Student at the University of Technology Sydney, Australia. She is currently with the School of Information Communication Engineering, Beijing Information Science and Technology University, Beijing, China. Her research interests include integrated sensing and communications, wireless federated learning, distributed computing, and privacy preservation.
\end{IEEEbiography}\vspace{-10 mm}

\begin{IEEEbiography}[{\includegraphics[width=1in,height=1.25in,clip,keepaspectratio]{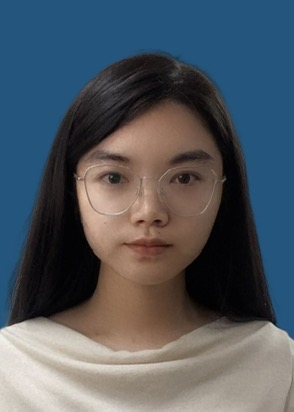}}]{Xi Yu} (Graduate Student Member, IEEE)
received the B.E. degree in communication engineering from Beijing University of Posts and Telecommunications (BUPT), China, in 2020. She is pursuing her Ph.D. with the School of Information and Communication Engineering at BUPT. Her research interests include multi-task semantic communication and privacy-preserving techniques.
\end{IEEEbiography}\vspace{-20 mm}

\begin{IEEEbiography}[{\includegraphics[width=1in,height=1.25in,clip,keepaspectratio]{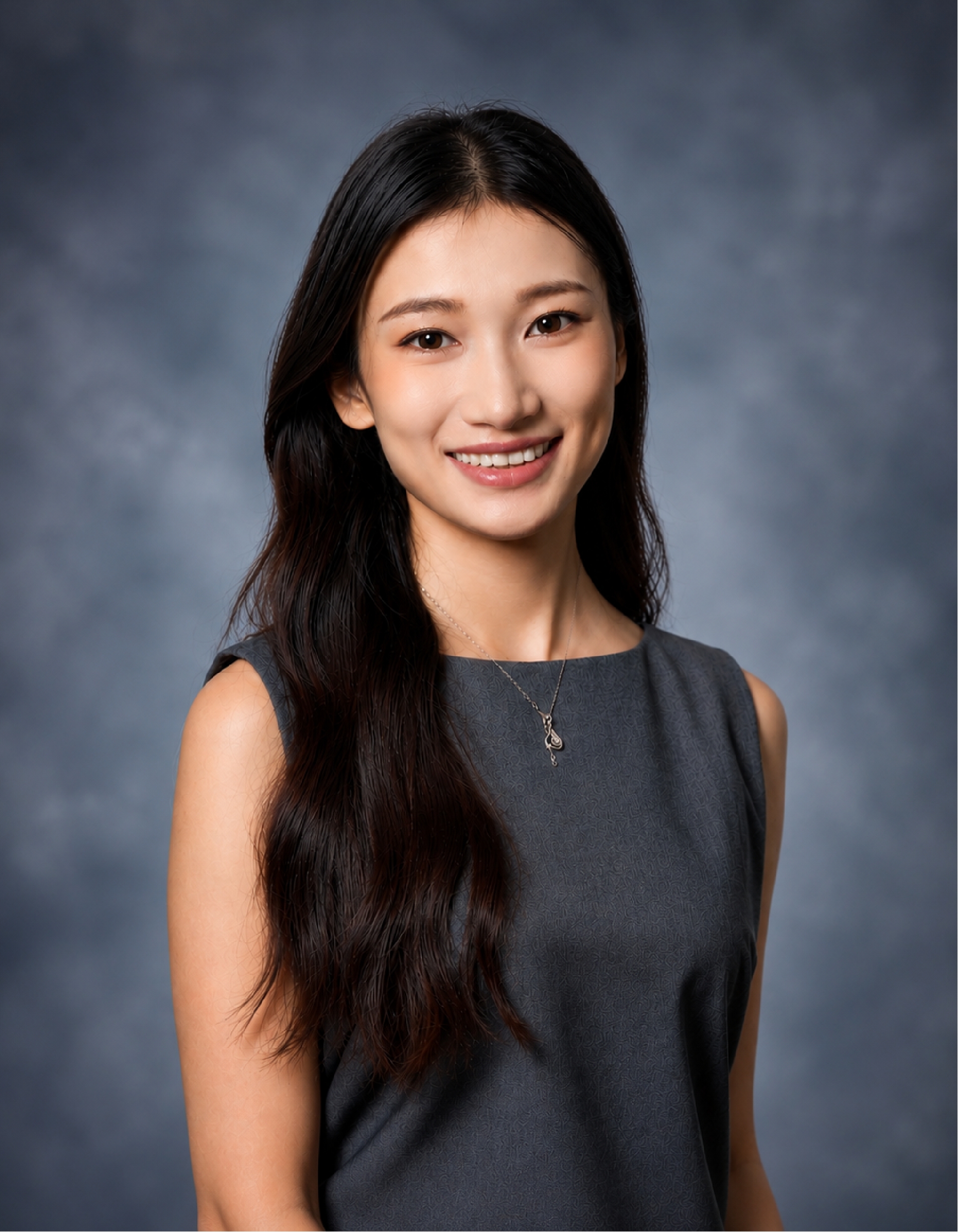}}]{Xiaoyu He} 
is currently pursuing the Ph.D. degree in Information and Communication Engineering at Beijing University of Posts and Telecommunications (BUPT), Beijing, China. From December 2025 to June 2026, she was a visiting Ph.D. student at the Singapore University of Technology and Design (SUTD), Singapore, under the supervision of Prof. Tony Q. S. Quek. Since August 2026, she has been a joint-training visiting Ph.D. student at the University of Bologna, Bologna, Italy, under the supervision of Prof. Paolo Bellavista. Her research interests include federated learning, machine unlearning, decentralized federated learning (D-FL), personalized learning, adaptive model pruning, reinforcement learning, non-cooperative game theory, multi-hop routing optimization, and resource-constrained wireless networks.
\end{IEEEbiography}\vspace{-20 mm}

\end{document}